\documentclass{article}

\usepackage{fullpage}
\usepackage{blindtext}
\usepackage{booktabs} 
\usepackage{makecell} 
\usepackage[round]{natbib}
\usepackage{url}
\usepackage{wrapfig} 
\usepackage{caption}

\usepackage{amssymb}
\usepackage{mathtools}
\usepackage{amsthm}
\usepackage{nameref}
\usepackage{hyperref}
\usepackage{aliascnt}
\usepackage[toc,page]{appendix}

\usepackage{amsmath}
\usepackage{amsfonts}
\usepackage{graphicx}
\usepackage[svgnames,dvipsnames]{xcolor}
\usepackage{mathrsfs,bm}
\usepackage{amssymb}
\usepackage{titlesec}
\usepackage{tikz-cd}
\usepackage[most]{tcolorbox}
\usepackage{etoolbox}
\usepackage{ stmaryrd }
\usepackage[T1]{fontenc}
\usepackage[shortlabels]{enumitem}

\usepackage[english]{babel}
\usepackage[utf8]{inputenc}
\usepackage{fancyhdr}
\usepackage{mathtools}
\usepackage{microtype}
\usepackage[normalem]{ulem}
\usepackage{stmaryrd}
\usepackage{wasysym}
\usepackage{multirow}
\usepackage{prerex}

\theoremstyle{plain}

\titleformat{\section}[block]{\color{black}\Large\bfseries}{\thesection}{1em}{}
\titleformat{\subsection}[hang]{\color{black}\large\bfseries}{\thesubsection}{1em}{}

\titleformat{\subsubsection}[hang]{\color{black}\large\bfseries}{\thesubsubsection}{1em}{}

\theoremstyle{plain}
\newtheorem{thm}{Theorem}
\newtheorem{theorem}[thm]{Theorem}

\newaliascnt{prop}{thm}

\newtheorem{prop}[prop]{Proposition}
\aliascntresetthe{prop}

\newaliascnt{lem}{thm}
\newtheorem{lem}[lem]{Lemma}
\newtheorem{lemma}[lem]{Lemma}
\aliascntresetthe{lem}

\newaliascnt{cor}{thm}

\newtheorem{cor}[cor]{Corollary}
\aliascntresetthe{cor}

\newtheorem{assumption}{Assumption}
\newtheorem*{assumption*}{Assumption}

\newtheorem{defn}[thm]{Definition}

\newtheorem*{conj*}{Conjecture}
\newtheorem*{defn*}{Definition}
\newtheorem*{note*}{Notation}
\newtheorem*{fact*}{Fact}
\newtheorem*{ques*}{Question}
\newtheorem*{exer*}{Exercise}
\newtheorem*{prob*}{Problem}

\newtheorem*{algo*}{Algorithm}

\renewcommand{\qedsymbol}{$\blacksquare$} 

\usepackage[most]{tcolorbox}
\newtcbtheorem{solb}{Solution}{
                lower separated=false,
                colback=white!80!gray,
colframe=white, fonttitle=\bfseries,
colbacktitle=white!50!gray,
coltitle=black,
enhanced,
breakable,
boxed title style={colframe=black},
attach boxed title to top left={xshift=0.5cm,yshift=-2mm},
}{solb}

\usepackage{cleveref}

\Crefname{defn}{Definition}{Definitions}
\Crefname{definition}{Definition}{Definitions}
\Crefname{rmk}{Remark}{Remarks}
\Crefname{prop}{Proposition}{Propositions}
\Crefname{thm}{Theorem}{Theorems}
\Crefname{theorem}{Theorem}{Theorems}
\crefalias{cor}{cor}
\Crefname{cor}{Corollary}{Corollaries}
\Crefname{lem}{Lemma}{Lemmas}
\Crefname{lemma}{Lemma}{Lemmas}
\Crefname{algo}{Algorithm}{Algorithms}
\Crefname{ex}{Example}{Examples}
\Crefname{answer}{Answer}{Answers}
\Crefname{ques}{Question}{Questions}
\Crefname{prob}{Problem}{Problems}
\Crefname{assumption}{Assumption}{Assumptions}
\Crefname{note}{Notation}{Notations}
\Crefname{fact}{Fact}{Facts}
\Crefname{exer}{Exercise}{Exercises}
\Crefname{conj}{Conjecture}{Conjectures}
\Crefname{claim}{Claim}{Claims}
\Crefname{figure}{Figure}{Figures}
\Crefname{subsection}{Subsection}{Subsections}
\Crefname{section}{Section}{Sections}
\Crefname{table}{Table}{Tables}
\crefformat{equation}{(#2#1#3)}
\Crefname{algocf}{Algorithm}{Algorithms}
\crefname{algocfproc}{Algorithm}{Algorithms}

\usepackage[algo2e, vlined, ruled, linesnumbered]{algorithm2e}

\SetCommentSty{mycommfont}

\usepackage{etoolbox}  
\makeatletter
\patchcmd{\algorithmic}{\addtolength{\ALC@tlm}{\leftmargin} }{\addtolength{\ALC@tlm}{\leftmargin}}{}{}
\makeatother

\newcommand{\nonl}{\renewcommand{\nl}{\let\nl}}

\SetKwRepeat{Do}{do}{while}

\newcommand\numberthis{\addtocounter{equation}{1}\tag{\theequation}}

\crefalias{AlgoLine}{line}%
\makeatletter
\let\cref@old@stepcounter\stepcounter
\def\stepcounter#1{%
  \cref@old@stepcounter{#1}%
  \cref@constructprefix{#1}{\cref@result}%
  \@ifundefined{cref@#1@alias}%
    {\def\@tempa{#1}}%
    {\def\@tempa{\csname cref@#1@alias\endcsname}}%
  \protected@edef\cref@currentlabel{%
    [\@tempa][\arabic{#1}][\cref@result]%
    \csname p@#1\endcsname\csname the#1\endcsname}}
\makeatother

\usepackage{mathtools}
\makeatletter
\newcommand{\mytag}[2]{%
  \text{#1}%
  \@bsphack
  \begingroup
    \@onelevel@sanitize\@currentlabelname
    \edef\@currentlabelname{%
      \expandafter\strip@period\@currentlabelname\relax.\relax\@@@%
    }%
    \protected@write\@auxout{}{%
      \string\newlabel{#2}{%
        {#1}%
        {\thepage}%
        {\@currentlabelname}%
        {\@currentHref}{}%
      }%
    }%
  \endgroup
  \@esphack
}
\makeatother

\definecolor{aqua}{rgb}{0.0, 1.0, 1.0}
\definecolor{caribbeangreen}{rgb}{0.0, 0.8, 0.6}
\definecolor{azure}{rgb}{0.0, 0.5, 1.0}
\definecolor{charcoal}{rgb}{0.21, 0.27, 0.31}

\hypersetup{
    colorlinks=true,
    breaklinks=true,
	citecolor=Mulberry,
	filecolor=black,
	linkcolor=azure,
	urlcolor=PineGreen,
	linkbordercolor=blue
}

\usepackage{subfigure}
\usepackage[subfigure]{tocloft}
\makeatletter
\def\clearwf{\par{\count@\c@WF@wrappedlines\zz}\par}

\def\zz{{%
\ifnum\count@>\@ne
\noindent\mbox{zz}\\%
\advance\count@\m@ne
\expandafter\zz
\else
\ifhmode\unskip\unpenalty\fi
\fi}}

\makeatother

\usepackage{joe/joe_macros}
\usepackage{joe/bandit_macros}
\usepackage{arxiv}
\usepackage{floatrow}
\usepackage{bookmark}
\usepackage{cancel}

\usepackage{apptools}
\AtAppendix{\renewcommand{\thesection}{\Alph{section}}
            \renewcommand{\thesubsection}{\thesection.\arabic{subsection}}
            \crefname{section}{Appendix}{Appendices}
            \Crefname{section}{Appendix}{Appendices}
}

\newfloatcommand{capbtabbox}{table}[][\FBwidth]

\title{
An Efficient Variant of One-Class SVM\\
with Lifelong Online Learning Guarantees
}

\usepackage{fancyhdr}
\renewcommand{\headrulewidth}{0.4pt}

\author{
	Joe Suk\thanks{Corresponding author; work done while at Columbia University.}\\
New York University\\
\href{mailto:j.suk@nyu.edu}{{ \texttt{j.suk@nyu.edu}}}%
\And
Samory Kpotufe\\
Columbia University\\
\href{mailto:samory@columbia.edu}{{ \texttt{samory@columbia.edu}}}%
}
\date{}

\begin{document}
\maketitle

\begin{abstract}
We study outlier (a.k.a., anomaly) detection for single-pass non-stationary streaming data.
In the well-studied offline or batch outlier detection problem, traditional methods such as kernel One-Class SVM (OCSVM) are both computationally heavy and prone to large false-negative (Type II) errors under non-stationarity.
To remedy this, we introduce SONAR, an efficient SGD-based OCSVM solver with strongly convex regularization.
We show novel theoretical guarantees on the Type I/II errors of SONAR, superior to those known for OCSVM, and further prove that SONAR ensures favorable lifelong learning guarantees under benign distribution shifts.
In the more challenging problem of adversarial non-stationary data, we show that SONAR can be used within an ensemble method and equipped with changepoint detection to
achieve {\em adaptive} guarantees, ensuring small Type I/II errors on each phase of data.
We validate our theoretical findings on synthetic and real-world datasets.
\end{abstract}

\section{Introduction}

We study the problem of outlier detection (a.k.a., novelty or anomaly detection) in a one-pass streaming setting where data arrive in a sequential or online manner. 
A key motivating example here is the real-life application of ``Internet of Things'' (IoT), where data consists of network traffic encodings from smart devices which carry signals of unusual or malicious activity.
In such applications, one aims to detect anomalous activity for security and threat identification.
For deployment of algorithms in the IoT problem, a common goal is to design memory and computation efficient algorithms, without needing to store or analyze large datasets, as typical smart devices are resource constrained.

In real-life applications, the underlying environment changes over time due to seasonalities or because streamlined software needs to be deployed across different modalities.
This necessitates procedures which adapt to non-stationary environments, or ensure {\em lifelong learning}, transferring knowledge across different data distributions.
While supervised transfer learning and non-stationary online learning have been well-studied in the literature, the particular problem of {\em non-stationary online outlier detection} has received less attention.
Additionally, the outlier detection problem poses unique challenges beyond supervised classification settings as the learner may not have access to outlier data, which can be adversarial in nature and change as the normal data distribution changes.

As is standard in works on outlier detection, our goal is to design procedures which adhere to the Neyman-Pearson framework: ensure small Type I error, or error of misclassifying normal data, while also maintaining small Type II error on unknown outlier data.
Our work rigorously proves such guarantees for a broad class of normal and outlier data distributions with supporting verifications of our theory on both synthetic and real-world IoT datasets.

We focus on a class of procedures based on the One-Class Support Vector Machine (OCSVM) \citep{scholkopf99} which is a common technique for offline outlier detection inspired by Support-Vector-Machines (SVMs) from supervised learning.
In particular, OCSVM has found widespread usage for outlier detection in IoT settings \citep{shilton2015dp1svm,lee2016packet,mahdavinejad2018machine,al2020unsupervised,razzak2020randomized}.

A first attempt at handling the streaming problem is to note OCSVM's objective function has a penalty-based formulation amenable to stochastic gradient descent (SGD).
However, since decision boundaries between normal and outlier data are unlikely to be linear for complex datasets, it's standard to use a kernel trick within OCSVM to embed the data into an infinite-dimensional reproducing kernel Hilbert space (RKHS).
This in turn poses several challenges for the penalty-based formulation, as the gradient updates of SGD depend on the gram matrix of the entire dataset, which is inaccessible in a streaming setting.

Instead, using Random Fourier Features \citep{rahimi07} for kernel approximation allow us to work in a finite-dimensional Euclidean approximation of the RKHS, where gradient updates can be made with one-pass streaming.

However, an additional challenge for theoretical analysis is that the objective function of OCSVM is not strongly convex, disallowing standard convergence analyses of the SGD iterates to the population minimizer.
This makes it tricky to conduct a rigorous analysis of the Type I and II errors of the associated SGD iterates.


We resolve this issue by introducing a new strongly convex variant of the OCSVM objective which is also amenable to SGD when using Random Fourier Features.
Going further, we find that our new objective's solution in fact admits stronger lifelong learning guarantees than the original OCSVM objective, even under unknown changing data distributions.
Our contributions are as follows:


\subsection{Contributions}

\begin{enumerate}[(a)]
	\item We first propose a new regularized objective which modifies the well-known One-Class SVM (OCSVM) optimization problem.
	We establish novel theoretical guarantees on the optimizer of this objective in terms of Type I error and learned margin.
	\item We next analyze SGD combined with random Fourier features (which we dub \sonar or \bld{S}GD-based \bld{O}ne-class \bld{N}ovelty detection with \bld{A}pproximate \bld{R}BFs)  to solve the regularized objective proposed in (a).
	We prove the final iterate of \sonar has guarantees closely matching those of the population optimizer in (a) above.
	\item We next show that \sonar in fact has transfer learning guarantees on the Type I and II errors, namely that it can benefit from the Type I and II errors of the initial iterate.
	We show how such guarantees induce robustness against mild non-stationarity in the environment.
	\item Finally, we propose a modification \sonarc (SONAR with Changepoint Detection) of \sonar for non-stationary environments which consists of running a changepoint detection on base learners, each of which is a copy of \sonar restarted at different dyadic frequencies.
	We show that \sonarc automatically adapts to unknown changepoints and closely matches the guarantees of an oracle procedure which runs \sonar and restarts at changepoints.
	\item We present experimental results analyzing SGD-based OCSVM, \sonar, and \sonarc on synthetic and real-world data, validating our theoretical findings.
\end{enumerate}

\section{Related Works}

\paragraph*{Outlier Detection in Changing Environments.}
Prior approaches to outlier detection in non-stationary data streams consist of sliding-window techniques, ensemble methods, and data weighting \citep{Tan2011Fast,KrawczykWozniak2015,Salehi2018Survey}.
However, theoretical guarantees largely remain unknown for this challenging setting.

In the offline setting, recent works have studied the Neyman-Pearson framework for outlier detection under distribution shift, establishing theoretical guarantees under various types of shifts \citep{kalan24,Kalan2025NP,Kalan2025Transfer}.
However, these works assume knowledge of source and target data, obviating the need for changepoint detection.

\paragraph*{Works on OCSVM.}
For theoretical guarantees, \citet{scholkopf99} showed the first generalization error bounds for OCSVM, appealing to the classical margin-based analysis of SVM \citep{Bartlett99}.
Later, \citet{vert05} showed consistency results (i.e., that OCSVM asymptotically learns a level set), for a Gaussian kernel with sample-dependent bandwidth.
Related to our work, \citet{yang22} also study the use of kernel data compression (via Nystr\"{o}m sketches) to reduce the time and space evaluation complexity of OCSVM.


OCSVM in $d$-dimensional Euclidean space is also equivalent to the {\em Minimum Enclosing Ball (MEB) problem} \citep{Megiddo1983Linear}, a well-studied computational geometry problem.
In the one-pass setting, \citet{zarrabi06} achieve a $3/2$ approximation ratio on the radius of the enclosing ball.
This approximation factor was later improved with accompanying lower bounds showing any algorithm using $\poly(d)$ space must incur a constant approximation ratio \citep{agarwal10,chan14}.
In multi-pass streaming settings, works on the broader online coreset problem showed tighter approximation factors are possible \citep{Badoiu2008Optimal}.
\citet{krivosija19} study a {\em probabilistic MEB} problem in an RKHS, where the goal is to minimize expected distance to heterogeneous randomly sampled datapoints, achieving a $1+\eps$ approximation ratio using $O(dn \poly(1/\eps))$ total training time for multi-pass streaming data.


We note that in the one-pass streaming setting, a constant approximation ratio on the radius of the enclosing ball
yields a more conservative decision boundary and thus a larger Type II error than OCSVM.
To contrast, our method (\sonar) achieves a smaller Type II error than OCSVM for alternatives near the decision boundary.

\section{Setup and Background on OCSVM and Random Fourier Features}


\begin{wrapfigure}{r}{0.4\textwidth}
    \centering
    \vspace{-3.5em}
    \includegraphics[width=0.8\textwidth]{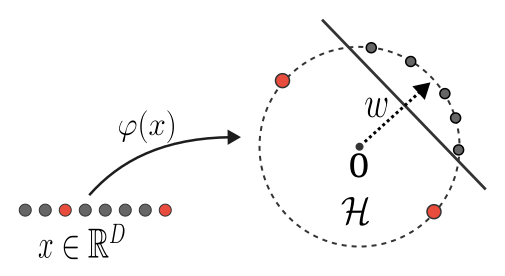}
    \caption{\small OCSVM maps data points $x \in \mb{R}^D$ as $\vphi(x)$ to a sphere in RKHS $\mc{H}$, inducing linear separation between normal points (gray) and yet unseen novel points (in red).}
    \label{fig:ocsvm}
\end{wrapfigure}

\subsection{Preliminaries}
Let $X_1,\ldots,X_T$ be our streaming data over $T$ steps, lying in domain $\mc{X} \subset \mb{R}^D$.
We use $\mc{H}$ to denote an RKHS with kernel $K(\cdot,\cdot)$ and kernel embedding $\vphi: \mc{X} \to \mc{H}$.
We assume the kernel is shift-invariant $K(x,y) = K(x-y)$ and that $K(x,x) = \langle \vphi(x),\vphi(x) \rangle_{\mc{H}}$ is a constant, as is the case for a Gaussian kernel.
This means $\vphi(x)$ lies in the unit sphere in $\mc{H}$ with respect to the RKHS norm.

\begin{defn}[Classification Criterion]
	A unit-norm RKHS element $w \in \mc{H}$ and threshold $c_0 \in \mb{R}$ are associated to a hyperplane in $\mc{H}$.
A test point $x \in \mb{R}^D$ is deemed an {\bf outlier} if $\vphi(x)$ and the origin are on the same side of the hyperplane, or if $\langle w, \vphi(x) \rangle = w(x) < c_0$.
This induces a classifier $h_w(x) := \pmb{1}\{ w(x) < c_0\}$ on outlier detection.
\end{defn}

\subsection{Background on OCSVM}
\paragraph*{Defining the OCSVM Objective.}
We first review the OCSVM objective, as originally formulated in \citet{scholkopf99,scholkopf01}.
OCSVM finds the maximum margin hyperplane separating the observed embedded data $\{\vphi(X_t)\}_{t=1}^T$ from the origin in the unit sphere of the RKHS $\mc{H}$.
This is equivalent to finding the minimal diameter spherical cap containing the data or:
\begin{equation}\label{eq:objective-radius}
	\min_{r > 0 ,w_c \in \mc{H}} r^2 \hspace{0.5em} \text{ s.t. }\hspace{0.5em} \forall t \in [T]: \|\vphi(X_t) - w_c \|_{\mc{H}}^2 \leq r^2 \, ,
\end{equation}
where $w_c$ denotes the center of the spherical cap in question.
We can also phrase the above as a margin maximization problem\footnote{The radius $r$ in \Cref{eq:objective-radius} corresponds to $\sqrt{2 - 2\rho}$ in terms of the offset $\rho$ in \Cref{eq:objective-normalized}.}:
\begin{equation}\label{eq:objective-normalized}
\max_{w,\rho}\, \rho  \hspace{0.5em} \text{ s.t }\hspace{0.5em} \forall t \in [T]: \langle w, \vphi(X_t) \rangle \geq \rho, \|w\|_{\mc{H} }= 1,
\end{equation}
where $w$ is the unit-norm normal vector to the separating hyperplane and $\rho$ is the threshold on the inner product for normal data.
The optimizer $w^*$ solving \Cref{eq:objective-normalized} geometrically corresponds to a hyperplane of maximal distance from the origin enclosing the data as depicted in \Cref{fig:ocsvm}.

\citet{scholkopf01} also consider the simpler objective without a norm constraint on $w$:
\begin{equation}\label{eq:objective-standard}
	\min_{w,\rho} \half \|w\|_{\mc{H}}^2 - \rho \suchthatshort \forall t \in [T]: \langle w, \vphi(X_t)\rangle \geq \rho \, .
\end{equation}
In fact the above
is equivalent to \Cref{eq:objective-normalized}, as the distance of a hyperplane induced by a pair $(w,\rho)$ from the origin is $\rho / \|w\|$ \citep[Prop 2]{scholkopf01}.
The reader is referred to Section 2.1 of \citet{tax2004} for further discussion on the equivalence of \Cref{eq:objective-standard} and \Cref{eq:objective-normalized}.

%

As is the case for SVMs, we can also formulate the soft version of \Cref{eq:objective-standard} with slack variables:
\begin{align}\label{eq:objective-slack}
	\min_{w,\rho,\xi} \half \| w\|_{\mc{H}}^2 - \rho + \frac{1}{\lambda \cdot T} \sum_{t=1}^T \xi_t  \nonumber \suchthatshort \langle w,\vphi(X_t) \rangle_{\mc{H}} \geq \rho - \xi_t, \forall t\in [T]: \xi_t \geq 0 \, . \numberthis
\end{align}
Here, the parameter $\lambda$ is user-set and, in theory, an upper bound on the fraction of outliers as we expect outliers to violate the constraint $\langle w, \vphi(X)\rangle \geq \rho$ a $\lambda$ proportion of the time.
Proposition 4 of \citet{scholkopf01} shows that $\lambda$
is asymptotically equal to the proportion of outliers with i.i.d. continuous data.


\paragraph*{Putting the OCSVM Objective in ERM Form.}
We note the optimal slack values in \Cref{eq:objective-slack} are
$
	\xi_t = (\rho - \langle w, \vphi (X_t) \rangle_{\mc{H}})_+ \, ,
$
meaning we can simplify \Cref{eq:objective-slack} using a penalty term:
\begin{equation}\label{eq:objective-unconstrained}
	\min_{w,\rho} \half \|w\|_{\mc{H}}^2 - \rho + \frac{1}{\lambda \cdot T} \sum_{t=1}^T \left( \rho - \langle w, \vphi(X_t) \rangle_{\mc{H}} \right)_+,
\end{equation}
Now, we've arrived at an objective which is in an empirical risk minimization form and is thus amenable to SGD in the streaming setting.
However, a major issue is that for infinite-dimensional RKHS's (e.g.,  that of a Gaussian kernel) we don't have access to the Hilbert space norm $\|\cdot\|_{\mc{H}}$.

In the batch setting, to get around this the kernel OCSVM objective is instead typically solved by either (a) using the dual formulation of \Cref{eq:objective-unconstrained} or (b) using the representer theorem to constrain our attention to solutions in a finite dimensional subspace of $\mc{H}$.
However, both of these approaches pose the same challenge in the streaming setting.

\paragraph*{Challenge of Streaming Setting.}
Going into more detail on the second approach, the representer theorem gives $w(x) = \sum_{t=1}^T \beta_t \cdot K(x, X_t)$ for coefficients $\beta_t$ so that $\|w\|_{\mc{H}}^2 = \sum_{i,j}^T \beta_i \beta_j K(X_i,X_j)$ \citep{chapelle07}.
Then, the objective \Cref{eq:objective-unconstrained} becomes:
\begin{equation}\label{eq:dual}
	\min_{\beta,\rho} \half \beta^T \bld{K}_T \beta + \frac{1}{\lambda T} \sum_{t=1}^T ( \rho - \bld{k}_t \cdot \pmb{\beta} )_+ - \rho,
\end{equation}
where $\bld{K}_T = (K(X_i,X_j))_{1 \leq i,j\leq T}$ is the $T \times T$ gram matrix, $\bld{k}_t$ is the $t$-th column of $\bld{K}_T$, and $\pmb{\beta} := (\beta_1,\ldots,\beta_T)$.
While \Cref{eq:dual} is also amenable to SGD, it is no longer sensible in the streaming setting as $\bld{K}_T$ and $\bld{k}_t$ depend on future datapoints, disallowing unbiased single-sample subgradient estimates.
The approach of solving \Cref{eq:objective-unconstrained} via the dual introduces an identical issue.
While we could solve \Cref{eq:dual} at each timestep $T=t$ thus reducing to the offline problem, but this would require prohibitive memory and runtime.

To get around this data dependency issue, we first curiously note this key issue does not arise for the simpler linear OCSVM problem with a Euclidean kernel.
In such a case, the primal \Cref{eq:objective-unconstrained} simplifies to:
\begin{equation}\label{eq:objective-euclidean}
	\min_{w,\rho} \half w^Tw - \rho + \frac{1}{\lambda \cdot T} \sum_{t=1}^T ( \rho - w^TX_t)_+ \, .
\end{equation}
Here, we have unbiased single-step subgradient estimate $(\rho - w^TX_t)_+$ only depending on the datapoint $X_t$.

Now, a non-linear kernel can be linearized using Random Fourier Features, allowing us to obtain an objective similar to \Cref{eq:objective-euclidean} for infinite-dimensional RKHS's.
We next review the details of Random Fourier Features (RFF).

\paragraph*{Review of Random Fourier Features.}
RFF approximates the feature map $\vphi(x)$ of a kernel $K(\cdot,\cdot)$ by means of a Euclidean projection $z:\mb{R}^D \to \mb{R}^d$ such that $z(x)^T z(y) \approx K(x,y)$.
We choose to use cos-sin pair RFF formulation of \citet{rahimi07}, which is defined by: 
\begin{align*}
	z(x) = \sqrt{\frac{2}{d}} \left( z_{\omega_1}(x), \ldots, z_{\omega_{2d}}(x) \right) = \sqrt{\frac{2}{d}} \left( \sin(\omega_1^Tx), \cos(\omega_1^Tx), \ldots, \sin(\omega_d^Tx), \cos(\omega_d^Tx) \right),
\end{align*}
where $\omega_1,\ldots,\omega_d$ are i.i.d. from the Fourier transform of $K(\cdot,\cdot)$ (which, recall is shift-invariant).
Then, the approximate kernel map is:
\begin{align*}
	z(x)^Tz(y) = \frac{1}{d}\sum_{j=1}^{d} \cos(\omega_j^T(x-y)) \approx \mb{E}_{\omega}[\cos(\omega^T(x-y))] = K(x,y) \, ,
\end{align*}
where the last equality holds by $K(x,y)=K(x-y)=\mb{E}_{\omega}[\exp(i\omega^T(x-y))]$ with expectation taken over the Fourier transform probability distribution of the kernel.

Now, the primal OCSVM objective \Cref{eq:objective-unconstrained} from before can be written as:
\begin{equation}\label{eq:obj-rff}
	\min_{w,\rho} \frac{\lambda}{2} \cdot w^Tw - \lambda \cdot \rho + \frac{1}{T} \sum_{t=1}^T \left( \rho - w^Tz(X_t) \right)_+ \, .
\end{equation}
This is a linear OCSVM problem for (projected) data $\{z(X_t)\}_{t=1}^T$ and the SGD update rules w.r.t. learning rate $\eta$ are:
\begin{align*}
	w_{t+1} &\leftarrow w_t - \eta \left( \lambda \cdot w_t -  z(X_t) \cdot \pmb{1}\{ \rho_t \geq w_t^T z(X_t) \} \right) \\
		\rho_{t+1} &\leftarrow \rho_t - \eta \left(  -\lambda + \pmb{1}\{ \rho_t \geq w_t^Tz(X_t) \} \right).
\end{align*}
However, the optimization problem \Cref{eq:obj-rff} is not strongly convex in the parameter $(w,\rho)$ (see \Cref{app:proof-ocsvm-not-sc}), making fast convergence guarantees unavailable.
Without such guarantees, it is unclear how to establish guarantees on Type I and II errors of the SGD iterates for \Cref{eq:obj-rff}.
To get around this, we define a new strongly convex variant of \Cref{eq:obj-rff}, which admits rigorous guarantees on both Type I and II errors.

\section{A New Strongly Convex Objective for Online Outlier Detection}

Since we choose to use RFF's to approximate our kernel, for simplicity we'll assume going forward that the data lie in the unit sphere $\mb{S}^{d-1}$ in $\mb{R}^d$ and satisfy the following margin assumption (which we show in \Cref{app:features} is justified for a Gaussian kernel for a sufficient number of RFF's $d$).

\begin{assumption}\label{ass:unit-sphere}
	Suppose the data domain $\mc{X}$ lies in the unit sphere $\mb{S}^{d-1} \subset \mb{R}^d$, and has margin from the origin of at least $1/2$, i.e. $r^* := \sup_{u \in \mb{S}^{d-1}} \inf_{x \in \mc{X}} \langle x, u\rangle \geq 1/2$.
\end{assumption}


We now present a strongly convex modification of \Cref{eq:obj-rff}.
For a probability measure $P_X$ on $\mc{X}$, we define objective:
\begin{equation}\label{eq:regularized-sc}
	F(w,\rho) := \frac{\|w\|_2^2 + \rho^2}{2} - \lambda \cdot \rho + \Exp\limits_{X\sim P_X}[ ( \rho - w^TX)_+ ].
\end{equation}

We next prove some basic properties of the optimizer to \Cref{eq:regularized-sc}.
The proofs of all presented results going forward are found in \Cref{app:proofs}.
First, we assert this objective is indeed strongly convex.

\begin{prop}\label{prop:obj-sc}
	$F(w,\rho)$ is $1$-strongly convex in $(w,\rho)$. 

\end{prop}

Next, we show the optimizer of \Cref{eq:regularized-sc}
has small Type I error, controlled by $\lambda$ similar to the OCSVM objective \Cref{eq:objective-slack} \citep[Prop. 4]{scholkopf01}.

\begin{thm}[Small Type I Error]\label{prop:small-type-1}
	Let $\lambda \in [0,1]$.
	The solution $(w_{\lambda}, \rho_{\lambda})$ to \Cref{eq:regularized-sc} satisfies
	$
		\mb{P}_{X\sim P_X}( \langle w_{\lambda}, X\rangle < \rho_{\lambda}) < \lambda \, .
	$
\end{thm}
On the other hand, our new objective \Cref{eq:regularized-sc} gives rise to a solution whose margin or distance of the separating hyperplane to the origin $r_{\lambda} := \frac{\rho_{\lambda}}{\|w_{\lambda}\|}$ is maximal, and in fact larger than that of OCSVM.
This implies that, in terms of margin, the solution to \Cref{eq:regularized-sc} generally has a larger rejection region for outliers.

\begin{theorem}[Large Margin Property]\label{prop:large-margin-property}
The solution $(w_{\lambda},\rho_{\lambda})$ to \Cref{eq:regularized-sc} satisfies $r_{\lambda} \geq r^*$. 
\end{theorem}

Thus, put together, our new objective's solution has the same Type I error bound as OCSVM, but a tighter acceptance region for inliers allowing for smaller Type II error against alternatives near the decision boundary.

\subsection{Solving Strongly Convex Objective with SGD}

\IncMargin{1em}
\begin{algorithm2e} 
	\caption{{\bld{S}GD-based \bld{O}ne-class \bld{N}ovelty detection with \bld{A}pproximate \bld{R}BFs (SONAR)}} \label{alg:sgd}
 \bld{Input}: step sizes $t\mapsto \eta_t$, initializer $(w_0,\rho_0)$, anticipated outlier proportion $\lambda \in [0,1]$.\\
 \For{$t=1,2,\ldots,T$}{
	 Observe $Z_t := \pmb{1}\{ \langle w_{t-1}, X_t \rangle \leq \rho_{t-1} \}$.\\
	 Update:\\
	 \Indp
		 $w_t \leftarrow w_{t-1} - \eta_{t-1} \cdot (w_{t-1} - X_t \cdot Z_t )$.  \label{line:w}\\
		 $\rho_t \leftarrow  \rho_{t-1} - \eta_{t-1} \cdot ( \rho_{t-1} - \lambda + Z_t )$. \label{line:rho}\\
	 \Indm
	}
\end{algorithm2e}
\DecMargin{1em}

We next show analogues of \Cref{prop:small-type-1,prop:large-margin-property} for the SGD iterates to solving \Cref{eq:regularized-sc} (see \Cref{alg:sgd}; named \sonar) given i.i.d. data from $P_X$.

We start by establishing a high-probability last-iterate convergence bound on the distance to the optimizer $(w_{\lambda}, \rho_{\lambda})$, which follows from more general convergence results \citep{harvey19a}.

\begin{lemma}[SGD Convergence Guarantee]\label{lem:high-prob-distance-SC}
	Suppose the data $X_1,\ldots,X_T$ are i.i.d. from some distribution $P_X$ and the initializer $(w_0,\rho_0)$ satisfies $\norm{w_0}_2 \leq 1$ and $\abs{\rho_0} \leq 1$.
	Let $(w_T,\rho_T)$ be the last iterate over $T$ steps of \Cref{alg:sgd} using step sizes $\eta_t := 1/t$.
	Then, we have with probability at least $1-\delta$:
	\begin{align*}
		\| (w_T, \rho_T) - (w_{\lambda} , \rho_{\lambda} ) \|^2 \leq O \left( \frac{\log(T) \log(1/\delta) }{T} \right).
	\end{align*}
\end{lemma}

As a consequence of this SGD convergence bound, we obtain analogues of \Cref{prop:small-type-1} and \Cref{prop:large-margin-property} for the iterates of SGD over $T$ steps.
We first present a generic Type I error bound in terms of the number of steps $T$ and $\lambda$.

\begin{thm}[Generalization Error Bound of Final Iterate]\label{thm:type-1-sgd}
	Under the same conditions as \Cref{lem:high-prob-distance-SC}, we have
	with probability at least $1-\delta$:
	\begin{align*}
		\mb{P}_X( \langle w_T, X \rangle < \rho_T) \leq \mb{P}_X\left( \langle w_{\lambda}, X\rangle < \rho_{\lambda} + O \left( \frac{\log^{1/2}(T) \cdot \log^{1/2}(1/\delta) }{\sqrt{T}} \right) \right).
	\end{align*}
\end{thm}

Next, we show that the Type I error is bounded by $\lambda$ when using a slightly smaller threshold $\rho_T \cdot (1-\eps)$ than prescribed by the iterate $(\rho_T,w_T)$.
This is akin to the classical generalization bound for the OCSVM solution \citep[Theorem 7]{scholkopf01}.

\begin{cor}[``Convergence'' of Generalization Error for Final Iterate]\label{cor:type-1-sgd}
	Under the same notation and conditions as \Cref{thm:type-1-sgd}, if $\frac{T}{\log(T)} \geq \Omega( (\eps \cdot \lambda)^{-2} \cdot \log(1/\delta) )$, then for any $\eps\in (0,1]$, with probability at least $1-\delta$:
	\[
		\mb{P}_X( \langle w_T, X \rangle < \rho_T \cdot (1-\eps)) < \lambda.
	\]
\end{cor}

Next, we show an analogue of \Cref{prop:large-margin-property} for the SGD iterates.

\begin{thm}[Large Margin Property for Final Iterate]\label{thm:margin-LB}
	Under the same notation and conditions as \Cref{thm:type-1-sgd}, we have
	with probability at least $1-\delta$:
	\[
		r_T := \frac{\rho_T}{\|w_T\|} \geq r_{\lambda} - O\left( \frac{\sqrt{ \log(T)\log(1/\delta) } }{\lambda \cdot \sqrt{T}} \right).
	\]
	Furthermore, if $\frac{T}{\log(T)} \geq \Omega( (\eps \cdot \lambda)^{-2} \cdot \log(1/\delta))$, then $r_T \geq r_{\lambda} \cdot (1 - \eps)$.
\end{thm}

\paragraph*{Comparing with Known Theoretical Guarantees for OCSVM.}
Even in the batch setting, the only previously known finite-sample high-probability generalization error bound to our knowledge is Theorem 7 of \citet{scholkopf01}, for the soft OCSVM problem in \Cref{eq:objective-slack}.
The bound therein scales like $\tilde{O} \left( \frac{\lambda}{r^*} + \frac{1}{(r^*)^2 \cdot T} \right)$, which is larger than our bound in \Cref{cor:type-1-sgd}.

In terms of learned margin, the population OCSVM objective only guarantees a learned margin of $r^*$ which we show in \Cref{thm:margin-LB} is smaller than \sonar's margin (approximating $r_{\lambda} \geq r^*$).

As it turns out, the OCSVM objective with slack \Cref{eq:objective-slack} also admits a solution with generalization error at most $\lambda$ and margin at least $r^*$ (see \Cref{app:ocsvm-lambda,app:ocsvm-large-margin}) which to our knowledge does not seem to be recorded in the literature.
In spite of this, as the corresponding objective is convex but not strongly-convex (see \Cref{app:proof-ocsvm-not-sc} for proof), SGD convergence rates will be slower of order $\tO(T^{-1/2})$ and, thus, error bounds looser than those of \sonar above.

\subsection{Comparing Complexity and Efficiency of SONAR vs Reduction to Batch OCSVM}

We show the training time, evaluation time (to classify a point), and memory of \sonar is superior to solving batch kernel OCSVM without RFF's at each step.
All our comparisons are done for the Gaussian kernel.

\paragraph*{Training Time.}
%
Each SGD update in \sonar takes $O(d)$ time for a total training time of $O(Td)$.
For the theoretically recommended choice of $d=O(D\log(D/\delta))$ (as justified in \Cref{app:features}), this gives a worst-case total runtime of $O(TD\log(D/\delta))$.

In contrast, solving the offline OCSVM at step $t$ using the seen data incurs higher train time.
Solving the dual convex quadratic program to \Cref{eq:objective-unconstrained} takes worst-case $O(t^3)$ time for a total time of $O(T^4)$.
Also, computing the gram matrix requires time $O(T^2 D)$.
If one uses SGD to solve the primal \Cref{eq:dual}, a total train time of at least $\Omega(T^2D)$ is still incurred just for gradient updates.
Using SMO to solve OCSVM also incurs $\Omega(T^2D)$ time to check KKT violations \citep{platt1998sequential}.


\paragraph*{Evaluation Time.}
\sonar requires $O(d)$ evaluation time at every step.
For the recommended choice of $d = O(D \log(D/\delta))$, this is smaller than that of batch OCSVM, which is worst-case $O(t D)$ at step $t$.

\paragraph*{Space Complexity.}
Offline OCSVM requires a memory of $O(TD+T^2)$ to store the gram matrix, which can be reduced using an SMO solver to $O(TD)$.
Meanwhile, \sonar only needs $O(d)$ memory, which is free of $t$.

\section{SONAR's Transfer Learning Guarantees}

We now tackle the more challenging non-stationary setting where the normal data distribution $P_{X_t}$ changes over time.
We first show that \sonar (\Cref{alg:sgd}) benefits from the past under favorable distribution shifts, without the need for any resets in learning.
We express such results as generalization error bounds and margin lower bounds akin to \Cref{cor:type-1-sgd} and \Cref{thm:margin-LB}, in terms of the quality of the iterate $(w_{t_0},\rho_{t_0})$ at a changepoint $t_0$.

\begin{thm}[Transferring Quality of Past Iterates]\label{thm:transfer-type-1}
	Suppose the data $X_{t_0},\ldots,X_t$ over period $[t_0,t]$ are i.i.d. from some distribution on $\mb{S}^{d-1}$.
	Let $(w_t,\rho_t)$ be the $t$-th iterate of \Cref{alg:sgd} using step sizes $\eta_t := 1/(t+1)$.
	Then, we have for steps $t > t_0$: 
	\begin{align}
		\mb{P} ( \langle w_t, X_t \rangle < \rho_t) \leq \underbrace{ \mb{P} \left( \langle w_{t_0}, X_t \rangle < \rho_{t_0} + \lambda \cdot \frac{t - t_0}{t_0}  \right)}_{\text{\normalfont\bf Benefit of Past}} \land \underbrace{\mb{P} \left( \langle w_{\lambda}, X_t \rangle < \rho_{\lambda} + \tO \left( \frac{1}{\sqrt{t - t_0}} \right) \right)}_{\text{\normalfont\bf Agnostic Guarantee}} \label{eq:transfer-1}\, .
	\end{align}
	Additionally, the $t$-th iterate's margin $r_t$ satisfies:
	\begin{align}
		r_t := \frac{\rho_t}{\norm{w_t}_2} \geq r_{t_0} + \sum_{s=t_0+1}^t \frac{ \lambda - Z_s \cdot (1 + \abs{r_{s-1}})}{ (s-1) \cdot \norm{w_{s-1} + Z_s X_s / (s-1) }_2}
		\numberthis\label{eq:margin-LB} \, .
	\end{align}

\end{thm}

\begin{wrapfigure}{r}{0.4\textwidth}
    \centering
    \vspace{-1em}
    \includegraphics[width=0.8\textwidth]{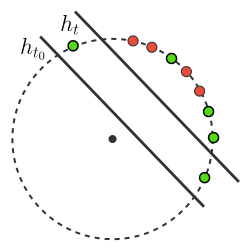}
    \caption{\small An example where the acceptance region may shrink from $t_0$ (normal data colored in green) to $t$ (normal data colored in red) with corresponding hyperplanes denoted by their classifiers $h_{t_0},h_t$.}
    \label{fig:ocsvm2}
    \vspace{-1em}
\end{wrapfigure}

The generalization error bound \Cref{eq:transfer-1} shows that for small $t-t_0$, we benefit from the generalization error of $(w_{t_0},\rho_{t_0})$, while for larger $t-t_0$, we can default to the agnostic guarantee of \Cref{thm:type-1-sgd}.

From the margin lower bound in \Cref{eq:margin-LB}, we see that $\lambda$ controls a tradeoff in Type I and II errors.
For steps where there is no misclassification $Z_s=0$ (i.e., low Type I error), the margin $r_t$ increases by an amount proportional to $\lambda$, thus reinforcing smaller Type II error for alternatives near the decision boundary.
On the other hand, misclassification $Z_s=1$ (more likely for large $\lambda$) lowers the margin.




Additionally, the lower bound \Cref{eq:margin-LB} ``transfers'' the margin from step $t_0$ to $t$ up to an additive error determined by the misclassifications from $t_0$ to $t$.
Thus, if the margin of the past iterate $r_{t_0}$ was already large and its Type I error on target data is small, we improve the margin on target data.

Additionally, \Cref{thm:transfer-type-1} has application to the stationary setting where data is i.i.d. from steps $1$ to $t$.
Letting $t_0=0$, we obtain more refined Type I error and margin bounds than those of \Cref{thm:type-1-sgd,thm:margin-LB}, depending on the quality of the initializer $(w_0,\rho_0)$.

However, in the face of adversarial distribution shifts, past iterates may not confer any benefits on Type I or II error, necessitating changepoint detection and restarts in learning.





\section{Adaptive Guarantees under Non-Stationarity}\label{sec:cpd}

Next, we design a procedure which adapts to any non-stationary environment, by ensuring similar-order small Type I error (\Cref{cor:type-1-sgd}) and large margin property (\Cref{thm:margin-LB}) within each stationary phase of time.
In particular, our aim as in other works on ensemble non-stationary online learning \citep{herbster95,haussler98,auer2002nonstochastic,hazan07,daniely15} is to mimic the performance of an oracle restarting variant of \sonar which resets its learning rate at changepoints.

For simplicity, let's suppose we know the total number of steps $T$ (we can circumvent this with a time-varying learning rate and doing the analysis along doubling periods).
The high-level idea is to run multiple copies of \sonar which independently reset their learning rates at different schedules.
We call each such instantiation a {\em base algorithm}.
We detect a change in the environment using the iterates of different base learners.
Intuitively, if the population minimizer $(w_{\lambda},\rho_{\lambda})$ remains fixed, then the iterates of all base learners will converge to $(w_{\lambda},\rho_{\lambda})$.
Thus, we detect a change in the environment if there is large discrepancy in the iterates beyond the SGD convergence threshold of \Cref{lem:high-prob-distance-SC}.
\Cref{alg:sonarc} captures the overall procedure.

\IncMargin{1em}
\begin{algorithm2e}
	\caption{{\sonar with \bld{C}hangepoint Detection (SONARC)}}
 \label{alg:sonarc}
 \bld{Input}: time horizon $T$, anticipated outlier proportion $\lambda$, error control $\delta$, changepoint detection threshold $C>0$.\\
 \For{$t=1,2,\ldots,T$}{
	 Observe $X_t$ and update each of the bases $\sonarbase(m)$ for $m \in \{1, \ldots, \floor{\log(T)}\}$.\\
	Let $(w_{t,m},\rho_{t,m})$ denote the most recent ``final iterate'' of $\sonarbase(m)$ before resetting its learning rate.\\
	If
	\begin{align}\label{eq:cpd}
		&\| (w_t, \rho_t) - ( w_{t,m}, \rho_{t,m}) \|_2^2 \geq C \cdot \frac{ \log(T) \cdot \log(1/\delta) }{2^m} \, ,
	\end{align}
	then restart the whole procedure with time horizon $T-t$. 
	}
\end{algorithm2e}
\DecMargin{1em}


We present two main theoretical guarantees for \sonarc.
We first show a {\em safety guarantee}: that w.h.p. a change is detected \Cref{eq:cpd} only if there is a change in $(w_{\lambda},\rho_{\lambda})$.

\begin{thm}[Safety of CPD]\label{thm:safety}
	If \Cref{eq:cpd} is triggered at time $t$ within \Cref{alg:sonarc}, then for a suitably large constant\footnote{$C$ must be a large enough constant, but is free of other problem parameters such as $\delta,\lambda, T$.} $C>0$ with probability $1-\delta$, there is a change in the population minimizer $(w_{\lambda}(P_{X_s}),\rho_{\lambda}(P_{X_s}))$ for some $s < t$.
\end{thm}

Our next result shows \sonarc in fact adapts to unknown changes and recovers guarantees similar to \Cref{cor:type-1-sgd} and \Cref{thm:margin-LB} for each phase.
Crucially, \sonarc does not need to know the timing or number of stationary phases, and so is fully adaptive to unknown non-stationarity.
Additionally, \sonarc only detects changes in the population minimizer $(w_{\lambda},\rho_{\lambda})$ of \Cref{eq:regularized-sc}.
In particular, \sonarc's guarantees in the next result are invariant to changes in the makeup of the normal data distribution $P_X$ which leave the decision boundary $(w_{\lambda},\rho_{\lambda})$ intact.
To contrast, using an off-the-shelf changepoint detector could lead to unnecessary restarts as we demonstrate experimentally in \Cref{subsec:off-the-shelf}.

\begin{thm}[Adapting to Unknown Phases]\label{thm:adapt-cpd}
	Suppose we run \Cref{alg:sonarc} on an environment with stationary phases $\{ \mc{P}_i\}_{i=1}^L$, each of length at least $\tilde{\Omega}(\log(1/\delta) \cdot (\eps \cdot \lambda)^{-2})$ and with margin $r_i^*$ (see \Cref{ass:unit-sphere}).
	Then, on each phase $\mc{P}_i$ where a restart does not occur, we have with probability $1- \delta \cdot |\mc{P}_i|$: for steps $t$ in the second half of $\mc{P}_i$:
	\begin{align*}
		\mb{P}_{X\sim \mc{P}_i}(\langle w_t, X\rangle < \rho_t \cdot (1-\eps)) \leq \lambda \andtextshort
		r_t \geq r_i^* \cdot (1-\eps) \, .
	\end{align*}
\end{thm}

Thus, \Cref{thm:adapt-cpd} states that for each phase, \sonarc either restart or maintains gurantees on Type I error and margin similar to an oracle restarting procedure.


\section{Conclusion/Discussion}

We introduced \sonar, an efficient SGD-based algorithm for streaming one-class classification that achieves provable guarantees on both Type I and II errors.
A key finding is that \sonar admits favorable transfer learning properties under benign distribution shifts -- that is, past iterates confer benefits on both Type I error and margin without requiring explicit knowledge of source and target domains.
For general non-stationarity, \sonarc extends \sonar with a principled changepoint detection mechanism that adapts to unknown distribution shifts while avoiding unnecessary restarts.

While our Type II error guarantees are expressed through margin bounds, future work could derive direct Type II error bounds for various outlier models.
It is also left open to study how labeled outlier feedback could be incorporated in \sonar's design principles and further improve guarantees.

\section*{Acknowledgements}

This work was supported by the National Science Foundation under Grant No. 2334997, ``Collaborative Research: EAGER: CPS: Data Augmentation and Model Transfer for the Internet of Things''.
We thank Yuyang Deng for help with processing the IoT datasets.
We also acknowledge computing resources from Columbia University’s Shared Research Computing Facility project, which is supported by NIH Research Facility Improvement Grant 1G20RR030893-01, and associated funds from the New York State Empire State Development, Division of Science Technology and Innovation (NYSTAR) Contract C090171, both awarded April 15, 2010

\bibliographystyle{joe}
\bibliography{bibs/bandit_general,
bibs/nonstat,
bibs/online,
bibs/duel,
bibs/contextual,
bibs/nonpar,
bibs/drift,
bibs/slow,
bibs/suk,
bibs/rot,
bibs/rl,
bibs/squarecb,
bibs/outlier
}

\begin{appendices}
\appendix
\section{A Lower Bound on the Number of Needed Random Fourier Features}\label[appendix]{app:features}

Here, we investigate \Cref{ass:unit-sphere} more closely and derive a lower bound $r_{\min}$ on the margin separating the support from the origin.
We first study the RBF kernel without the use of Random Fourier Features.
For an RBF kernel, we have that if the untransformed data $x$ lie in a set of diameter at most $B$, then:
\[
	K(x,x') := \exp\left( - \gamma \cdot \|x-x'\|_2^2 \right) \geq \exp(- \gamma \cdot B^2 ) \geq r_{\min} \iff B \leq \sqrt{\gamma^{-1} \log(r_{\min}^{-1})} \, .
\]
In particular, for untransformed data lying in a ball of diameter $1$ and $\gamma = 0.5$, the above second inequality is satisfied for $r_{\min} = 1/2$.
Then, $K(x,x') \geq r_{\min} = 1/2 > 0$ for all $x,x' \in \supp(P_X)$ implies the support of the transformed data lies in the interior of a fixed hemisphere since the angle between points in $\mc{H}$ is lower bounded by $1/2$.

We next show that $r_{\min}$ lower bounds the margin of the population OCSVM problem \Cref{eq:objective-normalized}.
Let $w^*$ be the normalized solution to \Cref{eq:objective-normalized}.
We then have:
\[
	r^* = \min_{x \in P_X} \langle \vphi(x), w^* \rangle_{\mc{H}} \geq \min_{x,x' \in P_X} \langle \vphi(x), \vphi(x') \rangle_{\mc{H}} \geq r_{\min} \, ,
\]
where the first inequality above holds because $w^*$ lies in the convex hull of the support vectors.

We next generalize these claims to the Random Fourier Feature approximation of the RBF kernel.
For RFF map $z:\mb{R}^D \to \mb{R}^d$ approximating the kernel embedding $\phi:\mc{X} \to \mc{H}$, we desire accuracy
\begin{equation}\label{eq:approx-kernel}
	|z(x)^Tz(x') - K(x,x')| \leq r_{\min}/2,
\end{equation}
to ensure that the transformed data in the RFF RKHS also lies in a spherical cap with angle at least $r_{\min}/2$ or:
\[
	\min_{x,x' \in \mc{X}} z(x)^Tz(x') \geq r_{\min}/2 \, .
\]
This will ensure there is a similar order margin in the RFF RKHS.
By a uniform convergence argument on the randomness of selecting features for $z$ \citep[Proposition 1]{sutherland2015rfferror}, \Cref{eq:approx-kernel} will hold w.p. $1-\delta$ for $d = \Omega\left( \frac{D}{r_{\min}^2} \log\left( \frac{D}{\delta \cdot r_{\min}^2} \right) \right)$ features where recall $D$ is the ambient dimension of the untransformed data $\mc{X}$.

\section{Proofs of Theoretical Results}\label[appendix]{app:proofs}

\subsection{Proof that OCSVM ERM-Form Objective is Not Strongly Convex}\label[appendix]{app:proof-ocsvm-not-sc}

We show the OCSVM objective of \Cref{eq:obj-rff} may in general not be strongly convex.
Consider a point-mass distribution at $X$.
We simplify the objective as:
\[
	f(w,\rho) := \frac{\lambda}{2} w^Tw - \lambda \rho + (\rho - w^TX)_+ \, .
\]
When $\rho \neq w^TX$ this objective is smooth in the parameters and has Hessian
\[
	\nabla^2 f = \begin{pmatrix} \lambda \cdot \Id_{d\times d} & \pmb{0}_d^T \\
		\pmb{0}_d & 0
	\end{pmatrix} \, ,
\]
which is not uniformly positive definite.
This means the objective $f(w,\rho)$ is not strongly convex in $(w,\rho)$. \hfill\qedsymbol

\subsection{Proof of Strong Convexity of \sonar Objective \Cref{eq:regularized-sc} (\Cref{prop:obj-sc})}

It suffices to show $(w,\rho) \mapsto -\lambda\cdot \rho + \mb{E}_{X\sim P_X}[ (\rho - w^TX)_+ ]$ is convex.
Because this latter function is the sum of an affine function and the function $(w,\rho) \mapsto \mb{E}_{X\sim P_X}[ (\rho - w^TX)_+ ]$, it suffices to show this expectation is convex in $(w,\rho)$.
To wit, note $(\rho,w) \mapsto (\rho - w^TX)_+$ is convex being the composition of a non-decreasing function with an affine (and thus convex) function.
Then, since expectations preserve convexity, the proof is finished. \hfill\qedsymbol

\subsection{Proof that Minimizer of \sonar Objective \Cref{eq:regularized-sc} has Small Type I Error (\Cref{prop:small-type-1})}\label[appendix]{app:proof-small-type-1}

First, we note a zero-norm choice of $w=0$ turns the objective \Cref{eq:regularized-sc} into a quadratic in $\rho$:
\[
	F(0,\rho) = \half \rho^2 - \lambda \rho + (\rho)_+ \, ,
\]
which has a minimal value of $0$ at $\rho = 0$ for $\lambda \in [0,1]$.
Then, we can rule out $w=0$ as a solution by showing that other choices of $(w,\rho)$ yield negative objective values.

Let's suppose $w\neq 0$.
We reparametrize the variables as $r := \rho / \|w\|_2$, $z := \|w\|_2$, and $v := w / \|w\|_2$.
We then rewrite the objective \Cref{eq:regularized-sc}, letting $L(r,v) := \mb{E}_{X\sim P_X}[ (r - v^TX)_+]$, we obtain:
\[
	F(r,z,v) = z^2 \cdot \left( \half + \half \cdot r^2 \right) - \lambda \, r \, z + z \cdot L(r,v).
\]
This is a convex quadratic in $z$ with feasible minimal value:
\[
	F(r,v) := - \frac{ ( - \lambda r + L(r,v))^2}{2 + 2r^2} \, .
\]
if $\lambda r > L(r,v)$.
Otherwise, $\inf_{z > 0} F(r,z,v) = 0$.
We can again rule out all such solutions for this latter case by demonstrating a choice of parameters with negative objective value.

Let $r^*$ be the support margin (\Cref{ass:unit-sphere}) and $v^*$ be the normalized solution to the population OCSVM objective \Cref{eq:objective-normalized}, i.e.:
\begin{equation}\label{eq:support-solution}
	(r^*,v^*) := \argmax_{v,r} \rho \suchthatshort \forall X \in \supp(P_X): \langle v, X \rangle \geq r , \norm{v}_{2} = 1 \, .
\end{equation}
Then, $r^*,v^*$ is feasible for minimizing $F(r,v)$ and has negative objective value of:
\[
	F(r^*,v^*) = - \frac{ ( - \lambda r^*)^2}{2 + 2 \cdot (r^*)^2 } \, .
\]
This allows us to rule out the zero-norm choice $w=0$ from before.

It remains to minimize $F(r,v)$ over $r,v$ such that $\lambda r > L(r,v)$.
First, note that $L(r,v)$ is convex in $(r,v)$ since $(r - v^TX)_+$ is convex, being the composition of a non-decreasing function with an affine function, and since expectations preserve convexity.
Next, we claim the subderivative
\[
	\partial_r L(r,v) = [ \mb{P}( v^TX < r) , \mb{P}( v^TX \leq r) ] \, .
\]
The above is found by computing the left and right derivatives of $L(r,v)$ with respect to $r$.
For the right-derivative, by bounded convergence theorem:
\[
	\lim_{h \downarrow 0} \frac{L(r+h,v) - L(r,v)}{h} =  \mb{E}_{X \sim P_X}\left[ \lim_{h\downarrow 0} \frac{(r+h - v^TX )_+ - (r - v^TX )_+}{h} \right] = \mb{P}(v^TX \leq r) \, .
\]
Similarly, for the left-derivative, again by bounded convergence:
\[
	\lim_{h \uparrow 0} \frac{L(r+h,v) - L(r,v)}{h} = \mb{E}[ \pmb{1}\{ v^TX < r \} ] = \mb{P}( v^TX < r) \, .
\]
Now, let $(r_{\lambda},v_{\lambda})$ be the minimizer of $F(r,v)$ for $r,v$ such that $\lambda r > L(r,v)$.
Next, note the set of $r$ such that $\lambda r > L(r,v)$ is a convex and open set as $L(r,v) - \lambda r$ is convex and continuous in $r$ for fixed $v$.
Then, by chain rule and first-order optimality condition for minimizing a convex function over an open convex domain:
\[
	0 \in \partial_{r} F(r_{\lambda} ,v_{\lambda} ) \implies \exists g \in [\mb{P}( v_{\lambda}^TX < r_{\lambda} ) , \mb{P}( v_{\lambda}^TX \leq r_{\lambda} ) ] :  g - \lambda = \frac{(L(r_{\lambda}, v_{\lambda}) - \lambda r_{\lambda} ) \cdot r_{\lambda}}{1 + r_{\lambda}^2} < 0 \, .
\]
This means we must have $\mb{P}_{X\sim P_X}(v_{\lambda}^TX < r_{\lambda}) \leq g < \lambda$. \hfill\qedsymbol

\subsection{Proof that Soft OCSVM Solution has Small Type I Error}\label[appendix]{app:ocsvm-lambda}

Here, we show that the population OCSVM objective \Cref{eq:objective-unconstrained} has a solution with generalization Type I error at most $\lambda$.
The proof is similar to \Cref{app:proof-small-type-1}.
We write the objective as
\[
	\min_{w,\rho} \half \|w\|_2^2 - \rho + \frac{1}{\lambda} \mb{E}_{X \sim P_X}[ (\rho - \langle w, X \rangle)_+ ].
\]
Now, reparametrizing $(w,\rho) \mapsto (w/\norm{w}_2, \rho/\norm{w}_2,\norm{w}_2) =: (v,r,z)$ in the same manner as \Cref{app:proof-small-type-1} yields an objective function of
\[
	F(r,z,v) = \frac{z^2}{2} - r \cdot z + \frac{z}{\lambda} L(r,v) \, ,
\]
for $L(r,v) := \mb{E}_{X\sim P_X}[ (r - v^TX)_+]$.
Then, we apply a similar quadratic analysis to find that all parameters $(r,z,v)$ for which $F(r,z,v) > \lambda$ are ruled out.
For fixed $r$ and $v$, $F(r,z,v)$ is a quadratic in $z$ with minimizer
\[
	\min_{z \geq 0} F(r,z,v) =: F(r,v) = \begin{cases}
		0 & \lambda r \leq L(r,v) \\
		- \half \left( r - \frac{L(r,v)}{\lambda} \right)^2 & \lambda r > L(r,v)
	\end{cases} \, .
\]
Now, the choice $(r^*,v^*)$ defined in \Cref{eq:support-solution} satisfies $L(r^*,v^*)=0$, $\lambda r^* \geq L(r^*,v^*)$ (since $r^* > 0$), and
\[
	F(r^*,v^*) = - \frac{ (r^*)^2 }{ 2 } < 0 \, .
\]
Since we've demonstrated a feasible negative objective value, this rules out all choices of $(r,v)$ for which $\lambda r \leq L(r,v)$.

Similar to \Cref{app:proof-small-type-1}, let $(r_{\lambda},v_{\lambda})$ minimize $F(r,v)$ over $r,v$ such that $\lambda r > L(r,v)$.
We then have that the first-order optimality condition for minimizing $F(r,v)$ over $r,v$ such that $\lambda r > L(r,v)$ implies
\[
	0 \in \partial_r F(r_{\lambda}, v_{\lambda}) \implies \exists g \in [ \mb{P}( v_{\lambda}^TX < r_{\lambda}) , \mb{P}( v_{\lambda}^TX \leq r_{\lambda}) ] : - \half \left( r_{\lambda} - \frac{L(r_{\lambda},v_{\lambda})}{\lambda} \right) \cdot \left( 1 - \frac{g}{\lambda} \right) < 0 \, .
\]
This implies $\mb{P}(v_{\lambda}^T X < r_{\lambda}) \leq g < \lambda$. \hfill\qedsymbol


\subsection{Proof that Minimizer of \sonar Objective \Cref{eq:regularized-sc} has Large Margin (\Cref{prop:large-margin-property})}\label[appendix]{app:proof-large-margin}

Using the notation of \Cref{app:proof-small-type-1}, again let $r := \rho/\norm{w}_2, z := \norm{w}_2$, and $v:= w/\norm{w}_2$ with $r_{\lambda},v_{\lambda}$ being the corresponding minimizers.
Also, recall $L(r,v) := \mb{E}_{X\sim P_X}[ (r - v^TX)_+]$ and $F(r,v) := - \frac{ (-\lambda r + L(r,v))^2}{2 + 2r^2}$ which is the minimal value of the objective over $z >0$ for fixed $r,v$ satisfying $\lambda r > L(r,v)$.
Let $(r^*,v^*)$ be the normalized solution to the OCSVM objective \Cref{eq:objective-normalized}.
We first note that since $(r_{\lambda},v_{\lambda})$ minimizes $F(r,v)$, we must have:
\[
	\frac{ ( - \lambda r_{\lambda} + L(r_{\lambda}, v_{\lambda}))^2}{ 1 + r_{\lambda}^2}  \geq \frac{ (\lambda r^*)^2}{1 + (r^*)^2}
\]
Let $G(r) := r/(1+r^2)^{1/2}$.
Then, taking the square root of both sides of the above display, we have one of:
\begin{align*}
	-\frac{L(r_{\lambda},v_{\lambda})}{\sqrt{1+r_{\lambda}^2}} + \lambda G(r_{\lambda}) \geq \lambda G(r^*) \qquad \text{ or } \qquad \frac{ L(r_{\lambda}, v_{\lambda})}{\sqrt{1 + r_{\lambda}^2}} \geq \lambda ( G(r_{\lambda}) + G(r^*))
\end{align*}
The second case cannot be true since we also have from \Cref{app:proof-small-type-1} that $L(r_{\lambda},v_{\lambda}) < \lambda r_{\lambda} = \lambda G(r_{\lambda}) \cdot (1 + r_{\lambda}^2)^{1/2}$.
The first case implies $G(r_{\lambda}) > G(r^*)$.
But, $G(r)$ is strictly increasing for positive $r$ so this implies $r_{\lambda} \geq r^*$. \hfill\qedsymbol

\subsection{Proof that Soft OCSVM Solution has Large Margin}\label[appendix]{app:ocsvm-large-margin}

Using the notation of \Cref{app:proof-ocsvm-not-sc}, again let $r := \rho/\norm{w}_2, z := \norm{w}_2$, and $v:= w/\norm{w}_2$ with $r_{\lambda},v_{\lambda}$ being the corresponding minimizers of the soft OCSVM objective \Cref{eq:objective-unconstrained}.
Also, recall $L(r,v) := \mb{E}_{X\sim P_X}[ (r - v^TX)_+]$ and $F(r,v) := - \half \left( r - \frac{L(r,v)}{\lambda} \right)^2$ which is the minimal value of the objective over $z >0$ for fixed $r,v$ satisfying $\lambda r > L(r,v)$.
Let $(r^*,v^*)$ be the normalized solution to the OCSVM objective \Cref{eq:objective-normalized}.
We first note that since $(r_{\lambda},v_{\lambda})$ minimizes $F(r,v)$ and since $L(r^*,v^*)=0$, we must have:
\[
	\left( r_{\lambda} - \frac{L(r_{\lambda}, v_{\lambda})}{\lambda} \right)^2 > (r^*)^2 \, .
\]
Now, since $r_{\lambda} > L(r_{\lambda}, v_{\lambda}) / \lambda \geq 0$, taking square root of both sides of the above display gives us $r_{\lambda} \geq r^*$. \hfill\qedsymbol

\subsection{Proof of SGD Convergence Guarantee (\Cref{lem:high-prob-distance-SC})}

The subgradient inequality for strongly convex function $F$ (namely, that $\pmb{0} \in \partial F(w_{\lambda}, \rho_{\lambda})$) gives us
\[
	\| (w_T, \rho_T ) - ( w_{\lambda}, \rho_{\lambda} ) \|^2 \leq 2 \cdot ( F(w_T, \rho_T) - F(w_{\lambda}, \rho_{\lambda})) \, .
\]
Thus, it suffices to bound the objective value gap of the iterate $(w_T,\rho_T)$ using the results of \citet{harvey19a}.

We first verify that our optimization problem satisfies the setting/assumptions of \citet{harvey19a}.
We claim that the noise of the subgradient estimate of $F(w,\rho)$ used in the update rules of \Cref{alg:sgd} have bounded norm.

\paragraph*{Subgradient Estimates are Unbiased and have Bounded Noise.}
Define $H_X(w,\rho) := (\rho - w^TX)_+$, which is the hinge-like risk appearing in the definition of $F(w,\rho)$ for a single datapoint $X$.
We then have
\[
	- X \cdot \pmb{1}\{ \rho \geq w^TX\} \in \partial_w H_X(w,\rho) \qquad \text{and} \qquad \pmb{1}\{ \rho \geq w^TX \} \in \partial_{\rho} H_X(w,\rho) \, ,
\]
are subgradients of $H_X(w,\rho)$.
Then, we claim $g_w := \mb{E}_{X\sim P_X}[ - X \cdot \pmb{1}\{ \rho \geq w^TX \} ]$ and $g_{\rho} := \mb{P}_{X\sim P_X}( \rho \geq w^TX)$ are subgradients of $\mb{E}_{X\sim P_X}[ H_X(w,\rho)]$.
This follows from the definition of subgradient:
\begin{align*}
	H_X(w,\rho) - H_X(w',\rho') &\geq \pmb{1}\{ \rho \geq w^TX\} \cdot (X,1)^T (w-w',\rho-\rho') \implies\\
	\mb{E}_{X\sim P_X}[ H_X(w,\rho) - H_X(w',\rho') ] &\geq \mb{E}_{X\sim P_X}[ \pmb{1}\{ \rho \geq w^TX\} \cdot (X,1)]^T (w-w',\rho-\rho') \, .
\end{align*}
Additionally, $ \norm{ g_w + X \cdot \pmb{1}\{ \rho \geq w^TX \} }_2^2 + \abs{ g_{\rho} - \pmb{1}\{ \rho \geq w^TX \} }^2 $ is bounded by a constant free of $w,\rho$ since $X$ has unit norm.
Thus, the subgradient estimates used in Lines~\ref{line:w} and \ref{line:rho} of \Cref{alg:sgd} are unbiased and have noise with bounded norm.

\paragraph*{Bounding Lipschitz Constant of Objective.}
Next, we claim $F(w,\rho)$ is Lipschitz in $(w,\rho)$ with $O(1)$ Lipschitz constant.
We first analyze the Lipschitz constant of the objective defined for a point-mass at $X$: $F_X(w,\rho) := \half (\norm{w}_2^2 + \rho^2) - \lambda \cdot \rho + (\rho - w^TX)_+$.
We first claim that so long as the initial iterate $(w_0,\rho_0)$ satisfies $\norm{w_0}_2 \leq 1$ and $\abs{\rho_0} \leq 1$, the SGD iterates $(w_t,\rho_t)$ stay within the convex set $\Theta := \{ (w,\rho) : \norm{w}_2 \leq 1 , \abs{\rho} \leq 1\}$.
This follows from writing the update rules:
\begin{align*}
	w_t &= w_{t-1} - \eta_{t-1}  (w_{t-1} - X_t \cdot \pmb{1}\{ \rho_{t-1} \geq \langle w_{t-1} , X_t \rangle \} ) = (1-\eta_{t-1} ) \cdot w_{t-1} + \eta_{t-1} \cdot X_t \cdot \pmb{1}\{ \rho_{t-1} \geq \langle w_{t-1} , X_t \rangle \} \\
	\rho_t &= \rho_{t-1} - \eta_{t-1}  ( \rho_{t-1} - \lambda + \pmb{1}\{ \rho_{t-1} \geq \langle w_{t-1} , X_t \rangle \} ) = (1-\eta_{t-1} ) \cdot \rho_{t-1} + \lambda \eta_{t-1} - \eta_{t-1}  \pmb{1}\{ \rho_{t-1} \geq \langle w_{t-1} , X_t \rangle \} \, ,
\end{align*}
whence we observe $w_t$ is a convex combination of vectors in the closed unit ball and $\rho_t$ is a convex combination of scalars in the interval $[-1,1]$.
Hence, the iterates $w_t$ and $\rho_t$ have norm at most $1$ throughout.

Then, it suffices to bound the subgradient norm on this compact parameter space $\Theta$.
The subdifferential of $F(w,\rho)$ is
\[
	\partial F_X(w,\rho) = (w, \rho - \lambda) + (-X, 1) \cdot \begin{cases}
		\{0\} & \rho - w^TX < 0\\
		[0,1] & \rho - w^TX = 0\\
		\{1\} & \rho - w^TX > 0
	\end{cases} \, ,
\]
with norm:
\begin{align*}
	\norm{ \partial F_X(w,\rho) }_2 &\leq \sup_{c \in [0,1]} \norm{ (w - c \cdot X , \rho - \lambda + c) }_2 \\
	&= \sup_{c \in [0,1]} \sqrt{ \norm{w - c \cdot X}_2^2 + (\rho - \lambda + c)^2 } \\
	&\leq \sup_{c\in [0,1]} \sqrt{ 1 + 1 + 2 \cdot ((\rho-\lambda)^2 + c^2)} \\
	&\leq \sqrt{2 + 2\cdot ((1  + \lambda)^2 + 1) } \, .
\end{align*}
This last bound is at most $\sqrt{12}$ for $\lambda \in [0,1]$.
This shows $F_X(w,\rho)$ is Lipschitz with Lipschitz constant $\sqrt{12}$.
It then follows $F(w,\rho) = \mb{E}_{X \sim P_X}[ F_X(w,\rho) ]$ has the same Lipschitz constant since by Jensen's inequality:
\[
	\abs{ F(w,\rho) - F(w',\rho') } \leq \mb{E}_{X\sim P_X} [ \abs{ F_X(w,\rho) - F_X(w',\rho') } ] \leq \sqrt{12} \cdot \norm{ (w,\rho) - (w',\rho')}_2 \, .
\]
Then, by Theorem F.1 of \citet{harvey19a}, we have w.p. $1-\delta$:
\[
	F(w_T, \rho_T) - F(w_{\lambda}, \rho_{\lambda} ) \leq O \left( \frac{\log(T) \log(1/\delta)}{T} \right).
\]
\hfill\qedsymbol

\subsection{Proof of Type I Error Bound of \sonar Final Iterate (\Cref{thm:type-1-sgd})}

First, we have:
\[
	\langle w_T, X \rangle < \rho_T \implies \langle w_{\lambda}, X \rangle - \norm{ w_T - w_{\lambda} }_2 < \rho_{\lambda} + \abs{ \rho_T - \rho_{\lambda}} \, .
\]
Then, plugging in bounds on $\norm{ w_T - w_{\lambda} }_2$ and $\abs{\rho_T - \rho_{\lambda}}$ from \Cref{lem:high-prob-distance-SC} finishes the proof. \hfill\qedsymbol

\subsection{Proof of Convergence of Type I Error for \sonar Final Iterate (\Cref{cor:type-1-sgd})}

Suppose $(w_T,\rho_T)$ is close to $(w_{\lambda},\rho_{\lambda})$ in the sense of $\|w_T - w_{\lambda}\|_2^2 + | \rho_T - \rho_{\lambda}|^2 \leq \Delta$.
First, by Cauchy-Schwarz and the fact that $X \in \mb{S}^{d-1}$, we have:
\[
	\langle  w_{\lambda} - w_T , X\rangle \leq \norm{ w_{\lambda} - w_T }_2 \cdot \norm{X}_2 = \norm{ w_{\lambda} - w_T }_2 \leq \Delta^{1/2} \, .
\]
Then
\begin{align*}
	\langle w_T, X \rangle &< \rho_T \cdot (1-\eps) \implies \\
	\langle w_{\lambda}, X \rangle + \langle w_T - w_{\lambda}, X \rangle &< \rho_{\lambda} + (\rho_T - \rho_{\lambda}) - \eps \cdot \rho_T \implies \\
	\langle w_{\lambda}, X \rangle - \Delta^{1/2} &< \rho_{\lambda} + \Delta^{1/2} - \eps\cdot \rho_T \implies \\
	\langle w_{\lambda}, X \rangle &< \rho_{\lambda} + 2 \Delta^{1/2} - \eps \cdot \rho_T.
\end{align*}
Then, it suffices to show $2 \Delta^{1/2} < \eps \cdot \rho_T$ whence the result follows from \Cref{prop:small-type-1}.
From our SGD convergence guarantee (\Cref{lem:high-prob-distance-SC}), we know that we can choose $\Delta := C_1 \cdot  T^{-1} \cdot \log(T) \cdot \log(1/\delta)$ for some constant $C_1 > 0$.

At the same time, by a lower bound on the population solution's margin (\Cref{prop:pop-norm-LB}; proven in \Cref{app:aux}), we have:
\[
	\rho_{\lambda} = \norm{w_{\lambda}}_2 \cdot r_{\lambda} \geq \frac{ \lambda \cdot r^* \cdot r_{\lambda} }{2} \, ,
\]
and that $\rho_T \geq \rho_{\lambda} - \Delta^{1/2}$.
Thus, it suffices to show
\[
	(C_1^{1/2} + \eps \cdot C_1^{1/2}) \cdot \frac{\log^{1/2}(T) \cdot \log^{1/2}(1/\delta)}{T^{1/2}} \leq \eps \cdot \frac{\lambda \cdot r^* \cdot r_{\lambda}}{2} \, .
\]
This is true for
\[
	T = \Omega( (r^* \cdot r_{\lambda})^{-2} \cdot \lambda^{-2} \cdot \eps^{-2} \cdot \log(1/\delta) \cdot \log(T) ) \, ,
\]
In particular, noting $r_{\lambda} \geq r^* \geq 1/2$ by \Cref{ass:unit-sphere}, we, for large enough $T$ in terms of $\delta, \lambda, \eps$:
\[
	\mb{P}_X( \langle w_T, X\rangle < \rho_T \cdot (1-\eps) ) \leq \mb{P}_X( \langle w_{\lambda}, X \rangle < \rho_{\lambda}  ) < \lambda\, ,
\]
where the last inequality is by \Cref{prop:small-type-1}. \hfill\qedsymbol

\subsection{Proof of Large Margin Property for \sonar Final Iterate (\Cref{thm:margin-LB})}

We first claim the solution $(w_{\lambda},\rho_{\lambda})$ to \Cref{eq:regularized-sc} satisfies $\| w_{\lambda}\|_2 \geq \lambda \cdot r^* / 2$ (\Cref{prop:pop-norm-LB}, proven in \Cref{app:aux}).
This gives us for some $C_1>0$ and by \Cref{lem:high-prob-distance-SC}:
\begin{align*}
	\| (\rho_T, w_T) - ({\rho}_{\lambda}, {w}_{\lambda}) \|_2 &\leq C_1 \cdot \frac{ \log^{1/2}(T) \cdot \log^{1/2}(1/\delta) }{T^{1/2}} \\
								&= C_1 \cdot \lambda \cdot r^* \cdot \frac{\sqrt{\log(T) \cdot \log(1/\delta)}}{\lambda \cdot r^* \sqrt{T}} \\
								&\leq 2 C_1  \cdot \|w_{\lambda}\|_2 \cdot \frac{\sqrt{\log(T) \cdot \log(1/\delta)}}{\lambda \cdot r^* \sqrt{T}},
\end{align*}
Let the above RHS be $\Delta$.
For $T = \Omega((r^* \cdot \lambda)^{-2} \cdot \log(T) \cdot \log(1/\delta))$, we have $\frac{\Delta}{2\norm{w_{\lambda}}_2} < 1$.
Then we use elementary algebraic manipulations (\Cref{lem:elementary-fraction} with \Cref{lem:margin-less-than-1} to assert $r_{\lambda} \leq 1$; both proven in \Cref{app:aux}) to get:
\[
	\frac{\rho_T}{\|w_T\|_2} \geq \frac{\rho_{\lambda} - \Delta}{\|w_{\lambda}\|_2 + \Delta} \geq r_{\lambda} - \frac{\Delta}{2 \|w_{\lambda}\|_2}.
\]
Furthermore, if $\frac{T}{\log(T)} = \Omega( (r^*)^{-4} \cdot (\eps \cdot \lambda)^{-2} \cdot \log(1/\delta))$, we have $\frac{\Delta}{2 \norm{w_{\lambda}}_2 } \leq \eps \cdot r_{\lambda}$ for any $\eps \in (0,1)$.
Since $r^* \geq 1/2$ by \Cref{ass:unit-sphere}, it suffices to take $\frac{T}{\log(T)} = \Omega( (\eps \cdot \lambda)^{-2} \cdot \log(1/\delta))$ whence $r_T \geq r_{\lambda} \cdot (1 - \eps) \geq r^* \cdot (1-\eps)$ by \Cref{prop:large-margin-property}. \hfill\qedsymbol

\subsection{Proof of Transfer Guarantees of \sonar (\Cref{thm:transfer-type-1})}

We first show the Type I error bound of \Cref{eq:transfer-1}.
First, we note that subgradients of the objective function $F(w,\rho,x) := \half ( \norm{w}^2 + \rho^2) - \lambda \cdot \rho + (\rho - w^TX)_+$ can be chosen as:
\begin{align*}
	\nabla_w F(w,\rho,X) &= w - X \cdot \pmb{1}\{ \rho \geq w^T X\} \\
	\nabla_{\rho} F(w,\rho,X) &= \rho - \lambda + \pmb{1}\{ \rho \geq w^T X\}.
\end{align*}
Then, we write update rules:
\begin{align*}
	w_t &= w_{t-1} - \eta_{t-1} \cdot (w_{t-1} - X_t \cdot \pmb{1}\{ \rho_{t-1} \geq \langle w_{t-1} , X_t \rangle \} ) = (1-\eta_{t-1} ) \cdot w_{t-1} + \eta_{t-1} \cdot X_t \cdot \pmb{1}\{ \rho_{t-1} \geq \langle w_{t-1} , X_t \rangle \} \\
	\rho_t &= \rho_{t-1} - \eta_{t-1} \cdot ( \rho_{t-1} - \lambda + \pmb{1}\{ \rho_{t-1} \geq \langle w_{t-1} , X_t \rangle \} ) = (1-\eta_{t-1} ) \cdot \rho_{t-1} + \lambda \eta_{t-1} - \eta_{t-1} \cdot \pmb{1}\{ \rho_{t-1} \geq \langle w_{t-1} , X_t \rangle \} \, .
\end{align*}
For shorthand, let $(w_0,\rho_0) = (w_{t_0}, \rho_{t_0})$.
Let $H_t := \prod_{s=t_0+1}^t (1-\eta_{s-1})$ and $J_{s,t} := \eta_{s-1} \cdot \prod_{u=s+1}^t (1-\eta_{u-1})$.
Inducting, we have closed-form formulas:
\begin{align}
	w_t &= H_t \cdot w_0 + \sum_{s=t_0+1}^t X_s \cdot \pmb{1}\{ \rho_{s-1} \geq \langle w_{s-1} , X_s \rangle \} \cdot J_{s,t} \numberthis\label{eq:closed-form-w} \\
	\rho_t &= H_t \cdot \rho_0 + \lambda \sum_{s = t_0+1}^t  J_{s,t} - \sum_{s=t_0+1}^t \pmb{1}\{ \rho_{s-1} \geq \langle w_{s-1} , X_s \rangle \} \cdot J_{s,t}. \numberthis\label{eq:closed-form-rho}
\end{align}
Next, letting $Z_t := \pmb{1}\{ \rho_{t-1} \geq \langle w_{t-1}, X_t \rangle\}$, we have for any $X \in \mc{X}$:
\begin{align}
	\langle w_t, X \rangle &< \rho_t \iff \nonumber\\
	H_t \cdot \langle w_0, X \rangle + \sum_{s = t_0+1}^t \langle X, X_s \rangle \cdot Z_s \cdot J_{s,t} &< H_t \cdot \rho_0 + \lambda \sum_{s = t_0+1}^t J_{s,t} - \sum_{s = t_0+1}^t Z_s \cdot J_{s,t} \iff \nonumber\\
	\langle w_0, X \rangle &< \rho_0 + \lambda \sum_{s = t_0+1}^t \frac{J_{s,t}}{H_t} - \sum_{s = t_0+1}^t \frac{J_{s,t}}{H_t} \cdot Z_s \cdot (1 + \langle X, X_s \rangle) \implies \label{eq:closed-form-initial}\\
	\langle w_0, X \rangle &< \rho_0 + \lambda \sum_{s = t_0+1}^t \frac{J_{s,t}}{H_t} ,\nonumber
\end{align}
where the last inequality follows from $1+\langle X, X_s \rangle \geq 0$ by virtue of $\supp(\mc{X})$ being contained in a fixed hemisphere of the unit sphere.

%
Now, recall our step size choice $\eta_t := \frac{1}{t+1}$.
We have by telescoping product $\prod_{u=t_0+1}^s (1 - \eta_{u-1}) = \frac{t_0}{s}$ and thus:
\[
	\sum_{s = t_0+1}^t \frac{J_{s,t}}{H_t} = \sum_{s=t_0+1}^t \frac{\eta_{s-1} }{\prod_{u=t_0+1}^s (1-\eta_{u-1}) } = \sum_{s=t_0+1}^t \frac{1}{s}\cdot \frac{1}{\frac{t_0}{s}} = \frac{t-t_0}{t_0}.
\]

Thus, we conclude
\[
	\langle w_t, X\rangle < \rho_t \implies \langle w_0, X \rangle < \rho_0 + \lambda \cdot \frac{t-t_0}{t_0}.
\]

We next show the margin lower bound for $r_t$ in \Cref{eq:margin-LB}.
Using the closed-form formulas for the iterates $w_t,\rho_t$, we have:
\begin{align*}
	r_t - r_{t-1} &= \frac{\rho_{t-1} - \frac{\eta_{t-1}}{1-\eta_{t-1}} \cdot ( Z_t - \lambda) }{ \norm{ w_{t-1} + Z_t \cdot X_t \cdot \frac{\eta_{t-1}}{1 - \eta_{t-1}} }_2 } - \frac{\rho_{t-1}}{\norm{w_{t-1}}_2 } \, .
%
\end{align*}
When $Z_t=0$, the above RHS simplifies to $\frac{\lambda}{(t-1) \cdot \norm{w_{t-1}}_2}$.
When $Z_t=1$, we further have
\begin{align}
	\frac{\rho_{t-1} - \frac{\eta_{t-1}}{1-\eta_{t-1}} \cdot ( 1 - \lambda) }{ \norm{ w_{t-1} + X_t \cdot \frac{\eta_{t-1}}{1 - \eta_{t-1}} }_2 } - \frac{\rho_{t-1}}{\norm{w_{t-1}}_2 } &= \frac{\rho_{t-1} \cdot \norm{w_{t-1}}_2 - \frac{1}{t-1} \cdot (1 - \lambda) \cdot \norm{w_{t-1}}_2 - \rho_{t-1} \cdot \norm{w_{t-1} + X_t /(t-1) }_2}{ \norm{w_{t-1}}_2 \cdot \norm{w_{t-1} + X_t / (t-1) }_2} \nonumber \\
																							     &= \frac{\rho_{t-1} \cdot (\norm{w_{t-1}}_2 - \norm{w_{t-1} + X_{t} / (t-1)}_2 ) - \frac{\norm{w_{t-1}}_2 }{t-1} \cdot (1 - \lambda)}{ \norm{w_{t-1}}_2 \cdot \norm{w_{t-1} + X_t / (t-1) }_2 } \, . \numberthis\label{eq:z1-bound}
\end{align}
Next, we have by triangle inequality and the fact that $\norm{X_{t}}_2 = 1$:
\[
	\rho_{t-1} \cdot ( \norm{w_{t-1}}_2 - \norm{w_{t-1} + X_{t} / (t-1) }_2 ) \geq - \frac{ \abs{ \rho_{t-1} } }{t-1} \, .
\]
Thus, the RHS of \Cref{eq:z1-bound} simplifies to:
\[
	- \frac{1 - \lambda + \abs{ r_{t-1} } }{(t-1) \cdot \norm{w_{t-1} + X_t/(t-1) }_2} \, .
\]
Summing over $t$ and telescoping, we get:
\[
	r_t \geq r_0 + \sum_{s=t_0+1}^t \frac{ \lambda - Z_s \cdot (1 + \abs{r_{s-1}})}{ (s-1) \cdot \norm{w_{s-1} + Z_s X_s / (s-1) }_2} \, .
\]

\subsection{Proof of \Cref{thm:safety}}

If there is no change and so a fixed optimizer $(w_{\lambda},\rho_{\lambda})$, then by \Cref{lem:high-prob-distance-SC} we have the high-probability SGD convergence guarantees: for some large enough $C>0$
\begin{align*}
	\|(w_{t,m}, \rho_{t,m}) - (w_{\lambda}, \rho_{\lambda}) \|_2^2 &\leq C \frac{\log(2^m) \cdot \log(1/\delta) }{2^m} \\
	\| (w_t,\rho_t) - (w_{\lambda}, \rho_{\lambda})\|_2^2 &\leq C \frac{\log(T) \cdot \log(1/\delta)}{t}.
\end{align*}
Since $t \geq 2^m$ whenever we check \Cref{eq:cpd}, the restart cannot be triggered as, by triangle inequality:
\[
	\norm{( w_{t,m},\rho_{t,m}) - (w_t,\rho_t) }_2^2 \leq 2C \frac{\log(T) \cdot \log(1/\delta)}{2^m} \, .
\]

\subsection{Proof of \Cref{thm:adapt-cpd}}

For the first phase, $i=1$, the desired guarantees already hold by \Cref{cor:type-1-sgd} and \Cref{thm:margin-LB} and a union bound over steps $t$ in the second half of $\mc{P}_1$.
We also won't unnecessarily restart during the first phase with probability $1-\delta \cdot |\mc{P}_1|/2$ by \Cref{thm:safety}.

Next, fix a phase $\mc{P}_i$ with $i\geq 2$, and let $I$ be a dyadic approximation of $\mc{P}_i$ such that $|I|=2^m$ for the smallest $m\in\mb{N}$ such that $2^m \geq |\mc{P}_i|/8$, and such that $I$ is contained entirely within the first quarter of $\mc{P}_i$.
Such an interval $I$ exists provided $|\mc{P}_i| \geq 2^4$ which is ensured by hypothesis since $|\mc{P}_i|$ is a sufficiently large constant.

Let $(w_I,\rho_I)$ be the final iterate on interval $I$ of the base algorithm $\sonarbase(m)$, which resets every $2^m$ steps.
Then, in the second quarter of $\mc{P}_i$ or for $t \in [t_i+|\mc{P}_i|/4, t_i + |\mc{P}_i|/2] \cap \mb{N}$ for $t_i $ being the first round of $\mc{P}_i$, if $(w_t,\rho_t)$ deviates too much from $(w_I,\rho_I)$, \Cref{eq:cpd} will be triggered and \Cref{alg:sonarc} will restart.

Otherwise, if \Cref{eq:cpd} is not triggered within $\mc{P}_i$, we claim that ``good convergence'' of iterates will be maintained on the second half, which will lead to the desired guarantees.
Let $(w_{\lambda}, \rho_{\lambda})$ be the population minimizer to \Cref{eq:regularized-sc} for phase $\mc{P}_i$.
Then, with probability at least $1-\delta$: for all $t \in \{ t \geq t_i\} \cap \mc{P}_i$:
\[
	\|(w_{t}, \rho_{t}) - (w_{\lambda}, \rho_{\lambda})\|_2^2 \leq 2 \left( \norm{ (w_t, \rho_t) - (w_I, \rho_I) }_2^2 + \norm{ (w_I,\rho_I) - (w_{\lambda}, \rho_{\lambda}) }_2^2 \right) \leq C_2 \cdot \frac{\log(T) \cdot \log(1/\delta)}{|\mc{P}_i|} \, ,
\]
for sufficiently large $C_2>0$.
This ``convergence guarantee'' allows to repeat the proofs of \Cref{cor:type-1-sgd} and \Cref{thm:margin-LB} for the iterates $(w_t,\rho_t)$ in the second half $t \in \{ t\geq t_i \} \cap \mc{P}_i$, given that $|\mc{P}_i| \gtrsim \log(1/\delta) \cdot \log(T) \cdot (\lambda \cdot \eps)^{-2}$.

\section{Auxilliary Lemmas for Proofs}\label{app:aux}

\begin{lemma}\label{prop:pop-norm-LB}
	The solution $(w_{\lambda}, \rho_{\lambda})$ to \Cref{eq:regularized-sc} satisfies $\| w_{\lambda}\|_2 \geq \lambda \cdot r^* / 2$.
\end{lemma}

\begin{proof}
	Recall the following notation from the proof of \Cref{prop:large-margin-property} in \Cref{app:proof-large-margin}.
	Let
	\begin{align*}
		L(r,v) &:= \mb{E}_{X\sim P_X}[ (r - v^TX)_+] \\
		G(r) &:= \frac{r}{(1+r^2)^{1/2}} \, .
	\end{align*}
	Then, from the results of \Cref{app:proof-large-margin}:
	\begin{align*}
		\frac{L(r_{\lambda},v_{\lambda})}{\sqrt{ 1 + r_{\lambda}^2}} &\leq \lambda \cdot ( G(r_{\lambda}) - G(r^*)) = \lambda \cdot \left( \frac{r_{\lambda}}{(1 + r_{\lambda}^2)^{1/2}} - \frac{r^*}{(1+(r^*)^2)^{1/2}} \right) \implies L(r_{\lambda} , v_{\lambda}) \leq \lambda \cdot ( r_{\lambda} - r^*),
	\end{align*}
	where in the last inequality we use the fact that $r_{\lambda} \geq r^*$ from \Cref{prop:large-margin-property}.
	Then, from \Cref{app:proof-small-type-1} we know the norm of the optimal $w_{\lambda}$ satisfies:
	\[
		\|w_{\lambda}\|_2 = \frac{ \lambda \cdot r_{\lambda} - L(r_{\lambda},v_{\lambda}) }{ 1 + r_{\lambda}^2 } \geq \frac{ \lambda r^*}{ 1 + r_{\lambda}^2 } \geq \frac{\lambda r^*}{2} \, .
	\]
\end{proof}

\begin{lemma}\label{lem:margin-less-than-1}
	The population minimizer $(w_{\lambda},\rho_{\lambda})$ of \Cref{eq:regularized-sc} satisfies $r_{\lambda} := \frac{\rho_{\lambda}}{\norm{w_{\lambda}}_2 } \leq 1$.
\end{lemma}

\begin{proof}
	Recall the objective
	\[
		F(w,\rho) = \half ( \norm{w}_2^2 + \rho^2) - \lambda \rho + \mb{E}_{X\sim P_X}[ (\rho - w^TX)_+ ] \, .
	\]
	We claim that for any $(w,\rho)$ such that $\rho > \norm{w}_2$, the parameter $(w,\norm{w}_2)$ has strictly smaller objective value.
	First, note since $\norm{X}_2=1$ for any $X$, by Cauchy-Schwarz:
	\[
		\rho - w^TX \geq \rho - \norm{w}_2 > 0 \, .
	\]
	Thus, $(\rho - w^TX)_+ = \rho - w^TX$ and also $(\norm{w}_2 - w^TX)_+ = \norm{w}_2 - w^TX$.
	Then,
	\[
		\mb{E}_{X\sim P_X}[ (\rho - w^TX)_+ - (\norm{w}_2 - w^TX)_+ ] = \rho - \norm{w}_2 \, .
	\]
	Then,
	\[
		F(w,\rho) - F(w,\norm{w}_2 ) = \half (\rho^2 - \norm{w}_2^2) - \lambda ( \rho - \norm{w}_2 ) + (\rho - \norm{w}_2 ) = (\rho - \norm{w}_2) \cdot \left( \frac{\rho + \norm{w}_2 }{2} + 1 - \lambda \right) \, .
	\]
	This last expression is strictly positive for $\lambda \leq 1$.
\end{proof}

\begin{lemma}\label{lem:elementary-fraction}
	Let $a,b,x,\eps > 0$ with $a \leq b$ and $\eps < 1$.
	Then,
	\[
		x \leq b \cdot \eps / 2 \implies	\frac{a - x}{b+x} \geq \frac{a}{b} - \eps \, .
	\]
\end{lemma}

\begin{proof}
	Since $a \leq b$, we have
	\[
		a + b \cdot (1 - \eps) \leq b \cdot (2 - \eps) \, .
	\]
	Then,
	\[
		\frac{b^2 \eps}{a + b \cdot (1-\eps)} \geq \frac{b^2 \eps}{b \cdot (2-\eps)} = \frac{b\eps}{2-\eps} \geq \frac{b \eps}{2} \geq x \, .
	\]
	This gives
	\[
		-ax - bx + b \eps x \geq -b^2 \eps \implies ab - bx \geq ab - b^2 \eps + ax  - b \eps x \implies b (a-x) \geq (a - b \eps) \cdot (b + x) \, .
	\]
	Rearranging this last inequality, we arrive at $\frac{a-x}{b+x} \geq \frac{a}{b} - \eps$.
\end{proof}

\section{Experiments on Synthetic Data}

\subsection{Implementation Details.}\label[appendix]{subsec:implementation}
We use $\lambda := 0.01$ and $\frac{D}{r_{\min}^2} \log\left( \frac{D}{\delta \cdot r_{\min}^2} \right)$ Random Fourier Features for $r_{\min} = 0.5$ and $\delta := \lambda/2$, following the theoretical bound of \Cref{app:features}.
This approximates an RBF kernel with $\gamma=0.5$.

In each experiment, we track the following over time $t$: (1) the {\em online Type I error} defined as the normalized mistake count $\frac{1}{t} \sum_{s=1}^t \pmb{1}\{ \langle w_s, \vphi(X_s) \rangle < \rho_s \}$, and (2) the current iterate's margin $r_t := \frac{\rho_t}{\norm{w_t}}$.
We report averages of these quantities over $20$ different runs, using random draws of RFF's.

For SGD, we use a time-varying AdaGrad-style learning rate, and also use Bottou's heuristic \citep{bottou2012stochastic} to further stabilize the learning rate for early steps.
The code is available at \texttt{\href{https://github.com/joesuk/SONAR}{https://github.com/joesuk/SONAR}}.

\subsection{Comparing SONAR and OCSVM on Stationary Environment}


\paragraph*{SONAR vs OCSVM on Stationary Environment.}
First, we analyze a simple stationary two-cluster environment and show \sonar and OCSVM have comparable performance in terms of Type I error and learned margin.
We generate $T=20000$ data points from a Gaussian mixture of clusters centered at $(-2,2)$ and $(2,-2)$ with standard deviation $0.3$.

From \Cref{fig1}, we see that \sonar achieves slightly better online Type I error and similar learned margin $r_t$ as SGD applied to the OCSVM objective \Cref{eq:obj-rff}, and that the final learned decision boundaries (shown for a sample run) are similar.

\begin{figure}[H]
	\includegraphics[width=\linewidth]{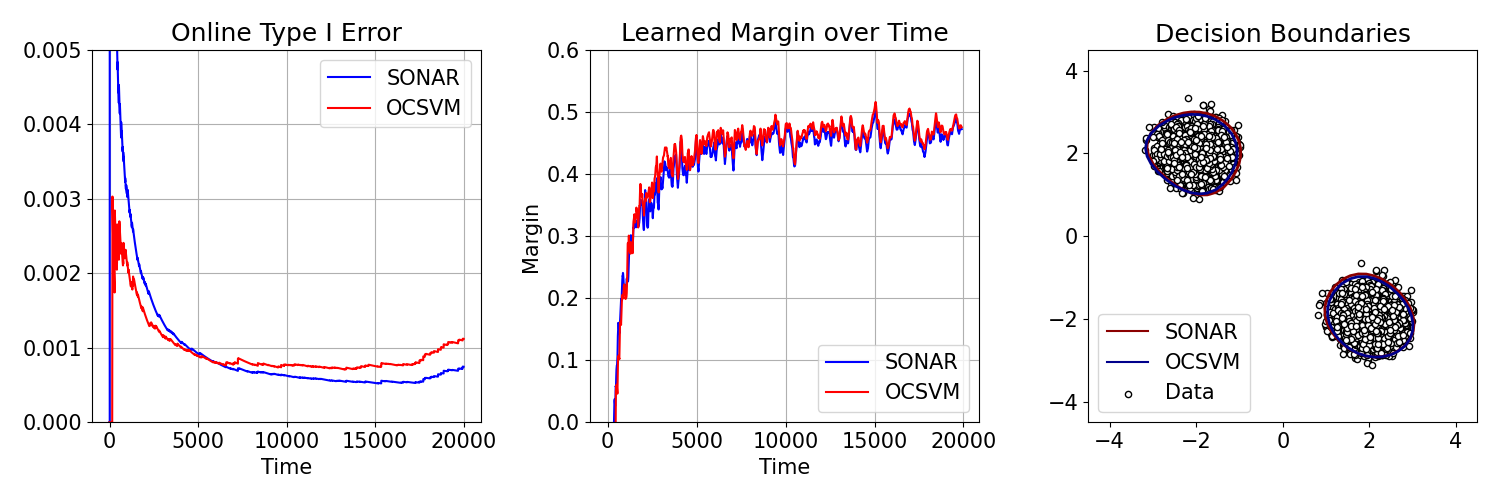}
	\caption{Online Type I Error, (smoothed) margin plot, and final decision boundaries of \sonar and SGD-OCSVM.}
	\label{fig1}
\end{figure}

\paragraph*{SONAR vs OCSVM on Mildly Non-Stationary Environment.}
We next study a non-stationary environment with $L=4$ phases, each of which consists of $10000$ datapoints drawn from a $2$-cluster Gaussian mixture model with standard deviation $0.3$ and centers (listed in order):
\begin{enumerate}[1.]
	\item $\{ (-2, 2), (2, -2) \}$.
	\item $\{ (-2, -2), (2, 2) \}$.
	\item $\{ (2, 2), (2, -2) \}$.
	\item $\{ (-2, 2), (-2, -2) \}$.
\end{enumerate}
As the number of changes $L$ is small here, it turns out that running a non-restarting procedure suffices for adaptivity.
We again run \sonar and SGD applied to the OCSVM objective \Cref{eq:obj-rff}.
From \Cref{fig2}, we see \sonar has superior online Type I error in later phases and learns a tighter decision boundary, as shown via a snapshot of acceptance regions.

\begin{figure}[H]
	\includegraphics[width=\linewidth]{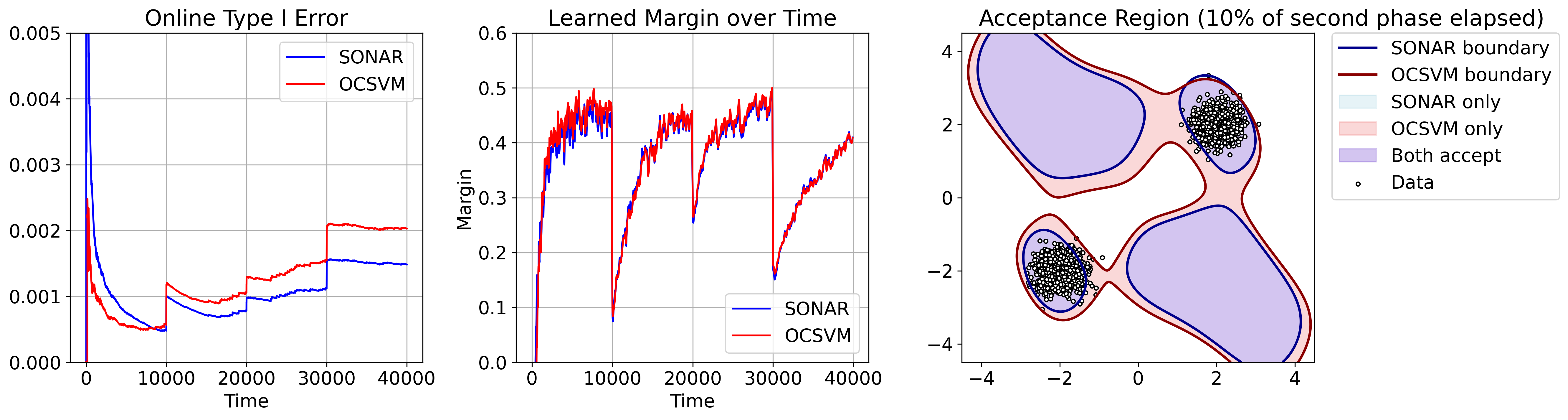}
	\caption{Online Type I Error, (smoothed) margin plot, and snapshot acceptance regions of \sonar and SGD-OCSVM.}
	\label{fig2}
\end{figure}

\subsection{Lifelong Learning Guarantees of \sonar: Benefiting from the Past}

We next study a non-stationary environment with $L=2$ phases, dubbed ``source'' and ``target''.
Each phase is drawn from a $2$-cluster Gaussian mixture model, with centers $(-2,-2)$ and $(2,2)$, with only the standard deviation changing from $0.6$ in source to $0.3$ in target.
We draw $10000$ datapoints per phase.
From \Cref{fig4}, we see \sonar maintains lifelong guarantees in terms of online Type I error and and maintins a high margin without requiring a restart.
The restarted \sonar is also unnecessarily conservative in learning the target boundary, as can be seen from depictions of the acceptance regions for both procedures with only $10\%$ of the target phase elapsed.

Altogether, these findings validate our theoretical findings in \Cref{thm:transfer-type-1} that benign distribution shifts admit fast learning rates on target from source.

%
%
%


\begin{figure}[H]
	\includegraphics[width=\linewidth]{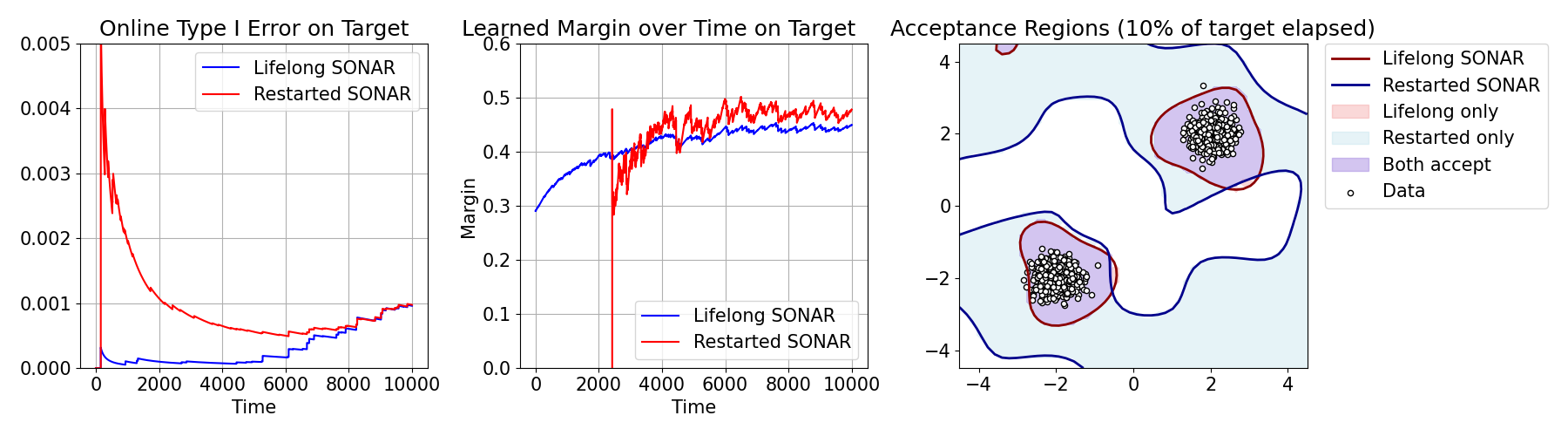}
	\caption{Online Type I Error, (smoothed) margin plot, and acceptance regions for at $10\%$ of the target phase elapsed.}
	\label{fig4}
\end{figure}

\subsection{Showing \sonarc Can Adapt to Non-Stationary Environments}\label[appendix]{subsec:exp-synthetic-nonstat}

We next study a non-stationary environment with $L=10$ phases each of which concists of $10,000$ datapoints from a $2$-cluster Gaussian mixture model with standard deviation $0.3$, and centers drawn uniformly at random from $[-5,5]^2$.
We run four procedures:
\begin{enumerate}[(a)]
	\item \sonar with no restarts, a.k.a. ``lifelong \sonar''.
	\item An oracle restarting variant of \sonar which resets its learning rate at the $9$ changepoints.
	\item \sonarc.
	\item A restarting \sonar equipped with an off-the-shelf multivariate Gaussian streaming change-point detector \citep[MDFocus]{changepoint_online,pishchagina2023online}. 
\end{enumerate}

For both \sonarc and \ttt{MDFocus}, we tuned the thresholds for changepoint detection by selecting the extremal threshold value which do not trigger in stationary data drawn from randomly generated $2$-cluster distributions of the same kind.

From the results in \Cref{fig5}, we see that the adaptive algorithms closely match the performance of the oracle procedure in terms of online Type I error and learned margin.
Meanwhile, the lifelong \sonar (depicted in blue) has worse Type I error and smaller margin.
\sonarc detected an average of $8.85$ changes, while \ttt{MDFocus} detected $9$ changes.

\begin{figure}[H]
	\includegraphics[width=\linewidth]{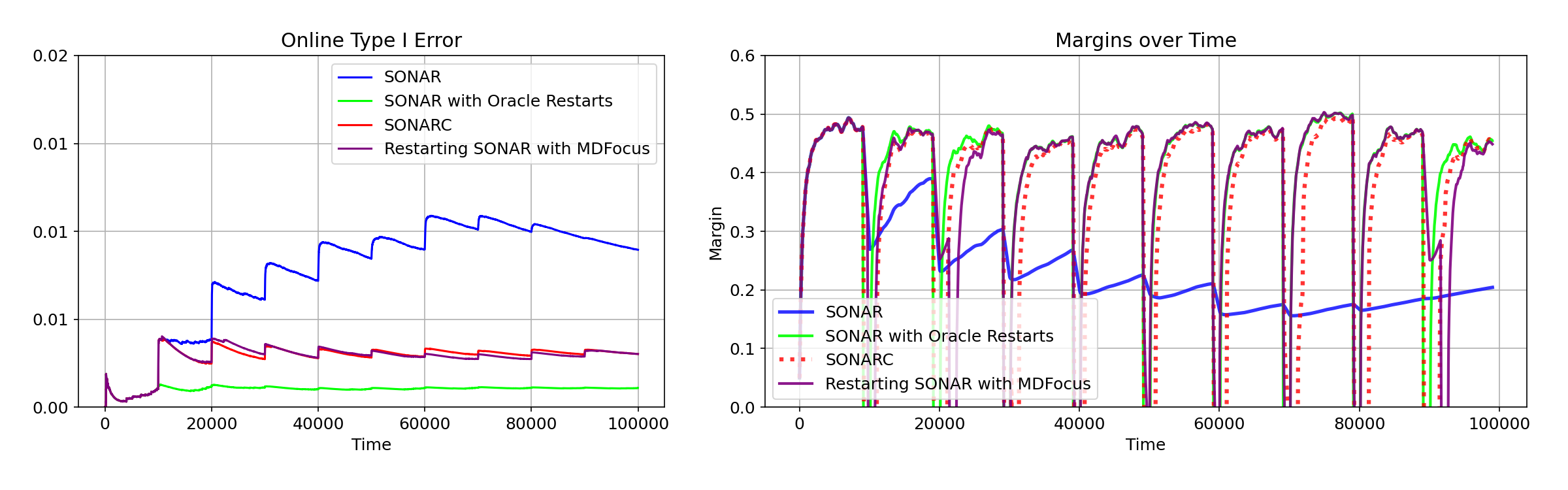}
	\caption{Online Type I Error and (smoothed) margin plot. }
	\label{fig5}
\end{figure}


\subsection{Demonstrating Off-the-Shelf Changepoint Detection can be Overly Conservative}\label[appendix]{subsec:off-the-shelf}

We next show that off-the-shelf changepoint detection (i.e., \ttt{MDFocus} used in \Cref{subsec:exp-synthetic-nonstat}) can restart unnecessarily.
Specifically, \sonarc's changepoint detector \Cref{eq:cpd} only triggers when the decision boundary of the data changes and is robust to shifts in the makeup of the support of normal data $P_X$.

Indeed, our experiment showcases such a scenario.
In each of two phases, we draw $5000$ datapoints from a $2$-dimensional Gaussian with identity covariance, truncated to the unit disk.
The phases only differ in the choice of the center of the Gaussian, set as the origin $(0,0)$ in the first phase and $(0.75,0)$ in the second.
This shift moves mass from the origin toward the new center.
A scatterplot of the data is found in the first subplot of \Cref{fig6}.

For both \sonarc and \ttt{MDFocus}, we tuned the thresholds for detection by selecting the most extremal threshold value which does not trigger in stationary data drawn independently from either of the two phase distributions.

Setting $\lambda = 0.01$, \ttt{MDFocus} detected one change whereas \sonarc detected zero changes, and both procedures secured a small final online Type I error of $0.0003$.
We also see that, due to \ttt{MDFocus}'s unnecessary restart, one must re-learn the margin and decision boundary, as seen in the second and third subplots of \Cref{fig6}.

\begin{figure}[H]
	\includegraphics[width=\linewidth]{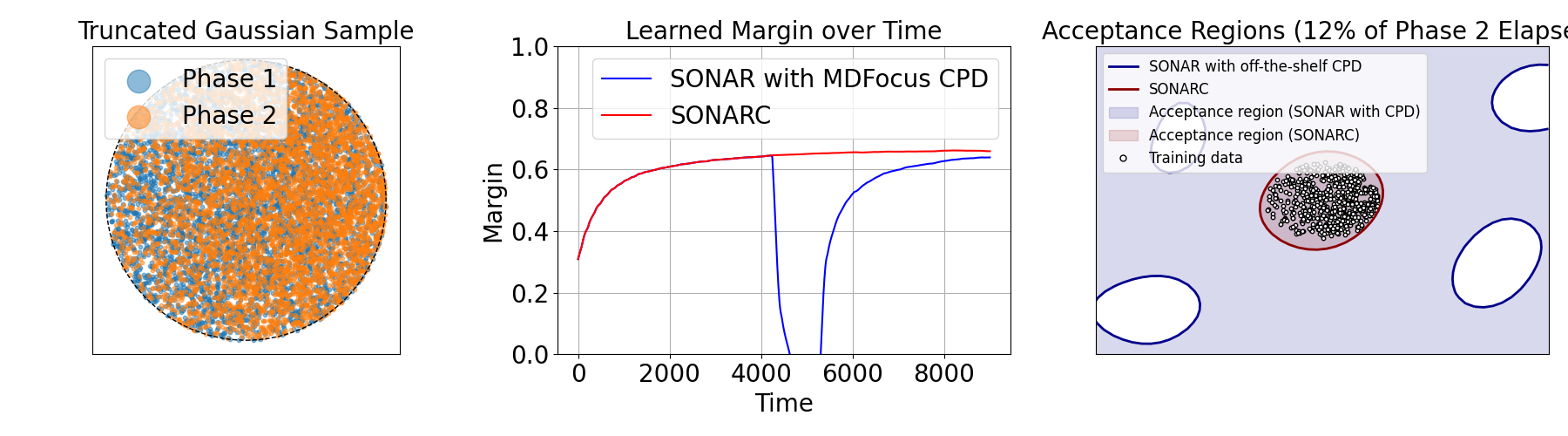}
	\caption{Scatterplot of data, margin plot, and and sample acceptance regions.}
	\label{fig6}
\end{figure}

\section{Experiments on Real-World Data}

We apply \sonar, \sonarc, and OCSVM applied to \Cref{eq:objective-unconstrained} to two real-world datasets, each of which has streaming data sorted in time and includes labelled outlier data.
We use the same setup as \Cref{subsec:implementation} with different choices of $\lambda$ for each dataset.

\paragraph*{Error Metrics.}
We focus on assessing the (normalized) {\em online cumulative Type I error} which is measured, at round $t$, as the normalized mistake count $\frac{1}{N_t} \sum_{s=1}^t \pmb{1}\{ \langle w_s, \vphi(X_s) \rangle < \rho_s \text{ and $X_s$ is normal datapoint} \}$, where $N_t$ is the number of normal data points seen till time $t$..
We also assess the {\em online cumulative Type II error} which is defined similarly: $\frac{1}{t-N_t}\sum_{s=1}^t \pmb{1}\{ \langle w_s, \vphi(X_s) \rangle \geq \rho_s \text{ and $X_s$ is an outlier}\}$.
Note that labeled outlier datapoints are {\em not} used for training.

In the non-stationary streaming setting, these online Type I and II error metrics are appropriate as the learner must adapt on-the-fly to sequential data and changing distributions.

\subsection{Evaluation on Skoltech Anomaly Benchmark (SKAB) Dataset}

The SKAB repository consists of multivariate time series data collected in 2020 from sensors installed on a water circulation loop testbed \citep[Skoltech Anomaly Benchmark]{skab}.
We concatenated the \ttt{valve1} and \ttt{valve2} datasets from SKAB, which both have $8$ features.
We set $\lambda = 0.005$ (or $0.5\%$ Type I error) and
standardized each datapoint $X_t$ using running estimates of the mean and variance up to time $t$.
We chose the threshold \Cref{eq:cpd} for changepoint detection in \sonarc from $\{ 10^{-n} \}_{n \in \mb{N}}$, setting it to be the largest value that detects at least one change.

From the first two plots of \Cref{fig:valve1}, we see that all three procedures perform similarly in terms of online Type I error, while \sonar and \sonarc have better Type I error than OCSVM.
Furthermore, \sonarc dramatically improves on online Type II error in early steps where it detected changes, at the cost of slightly higher Type I error (but well within range of the user-set $\lambda = 0.005$).
The improvement in Type II error is also reflected in the comparison of margin iterates $r_t$ in the third plot: \sonar and \sonarc reinforce a larger margin, aligning with our theory (\Cref{prop:large-margin-property,thm:margin-LB}).


\begin{figure}[H]
	\includegraphics[width=\linewidth]{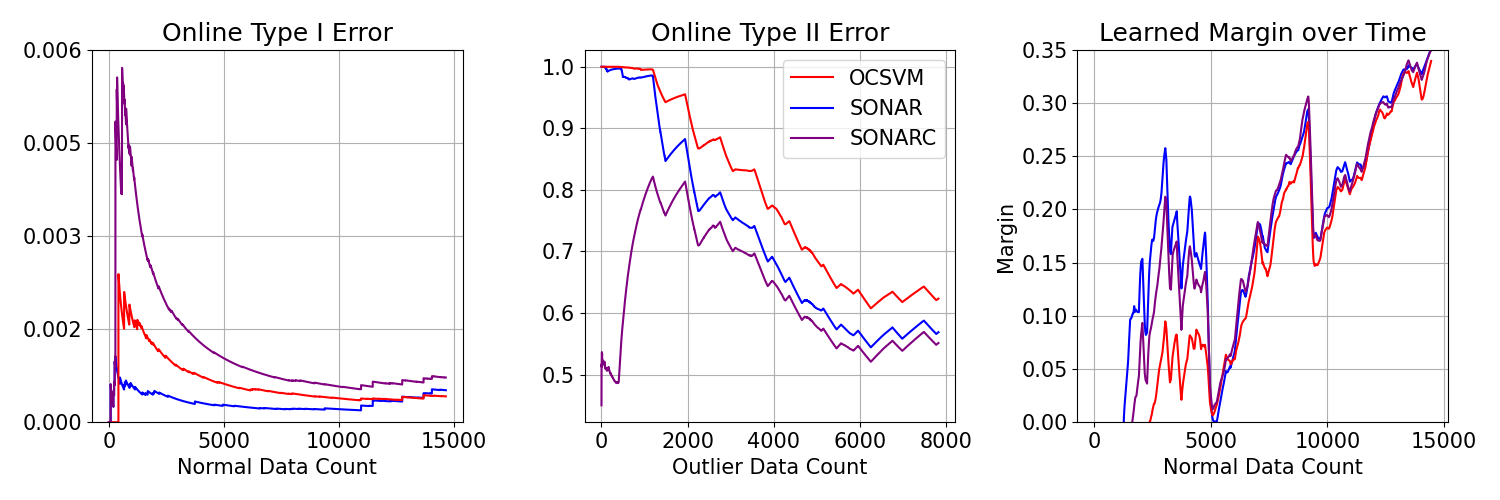}
	\caption{SKAB results: margin curves are smoothed averages.
	\sonarc detected on average $1.6$ changes.
	}
	\label{fig:valve1}
\end{figure}


\paragraph*{Comparison with Other Procedures.}
The official repository for SKAB reports a Kaggle leaderboard with batch Type I and II errors on held-out test datasets \citep{skab}.
The state-of-art consists of deep learning-based methods such as LSTM and Conv-AE as well as non-parametric methods such as MSET and Isolation Forest.
None of these methods are, without further modification, applicable to our one-pass streaming setting.
In spite of this, we find that our experiment's \sonarc has final online Type I error of $7.2 \times 10^{-4}$ and online Type II error of $0.551$, which compare favorably to many of the reported arts.

We do caution that this is not a full apples to apples comparison as the error metrics are computed differently.
Arguably, our online Type I and II errors are akin to test errors because evaluation is done at each step using an unseen datapoint.
In another sense, our online error metrics could be considered more challenging than simply computing errors on a test holdout dataset as the online metrics average over all steps, including earlier steps where convergence is slow.

The {\em final iterate's offline Type I or II error} on the entire dataset for \sonar is, respectively, $38\%$ and $26.6\%$, which yields an F1-score of $0.6$ for this problem.
This is also within the range of the state-of-the-art in the leaderboards.

\subsection{Aposemat IoT-23 Dataset}

The Aposemat IoT-23 repository is a labeled collection of network traffic data from real-world IoT devices, such as a Philips HUE smart LED lamp, an Amazon Echo home intelligent personal assistant and a Somfy Smart Door Lock \citep{aposemat23}.
Malware was executed on these devices using a Raspberry Pi, giving normal and outlier labels on flows.
Among the various captures in Aposemat IoT-23, we chose the \ttt{CTU-IoT-Malware-Capture-33-1} dataset for its abundance of normal data.
We processed the raw flow data using tools found in the \ttt{netml} library introduced by \citet{yang2020comparative}.
The 33-1 dataset has one feature after this processing.
As in our previous experiment on the SKAB dataset, we standardize each datapoint $X_t$ using running estimates of the mean and variance up to time $t$.

For $\lambda = 0.1$, we plot the online Type I and II error curves over observed normal/outlier data count, as well as the margin iterates.
From these plots we see that \sonar and OCSVM exhibit a tradeoff in that \sonar has smaller Type I error, while OCSVM has smaller Type II error.
We also implemented \sonarc on this dataset but it did not restart for any threshold values, indicating that the dataset is fairly stationary.
This can also be intuited from the margin plot which shows steady convergence to the margin for both algorithms.



\begin{figure}[H]
	\includegraphics[width=\linewidth]{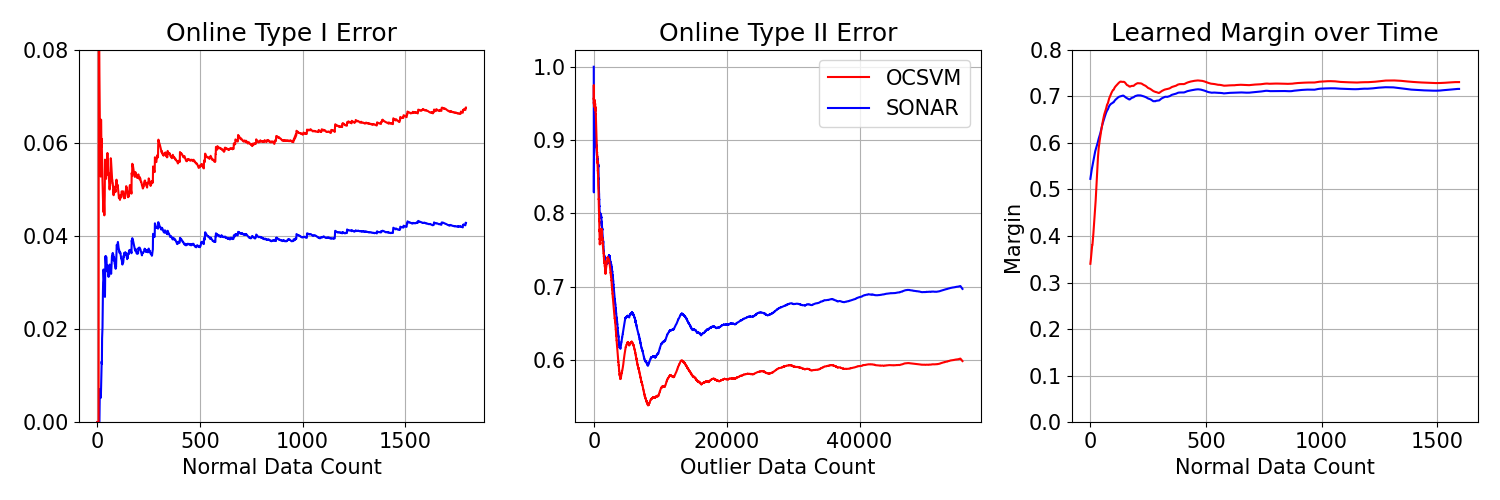}
	\caption{\ttt{CTU-IoT-Malware-Capture-33-1} results: margin curves are smoothed averages.}
	\label{fig:33-1}
\end{figure}

\end{appendices}

\end{document}